\documentclass[journal]{IEEEtran}
\usepackage[utf8]{inputenc}

\usepackage{hyperref}
\usepackage{graphicx}
\usepackage{amssymb}
\usepackage{amsmath}
\usepackage{amsthm}
\usepackage{mathtools}
\usepackage{bm}
\usepackage{color}
\usepackage[normalem]{ulem}

\usepackage{array}
\newcolumntype{C}[1]{>{\centering\arraybackslash}m{#1}}
\newcolumntype{R}[1]{>{\raggedleft\arraybackslash}m{#1}}

\newtheorem{assumption}{Assumption}
\newtheorem{theorem}{Theorem}   

\newtheorem{corollary}[theorem]{Corollary}
\newtheorem{lemma}[theorem]{Lemma}

\newtheorem{definition}{Definition}
\newtheorem{example}{Example}

\usepackage[ruled,vlined,linesnumbered]{algorithm2e}
\SetKw{KwStep}{step}
\SetKwInput{KwData}{Input}
\SetKwInput{KwResult}{Result}
\SetKwBlock{RepeatForever}{repeat}{}
\SetKwComment{Comment}{$\triangleright$\ }{}
\SetCommentSty{itshape}

\usepackage[capitalise]{cleveref}

\crefname{lemma}{lemma}{lemmas}
\Crefname{lemma}{Lemma}{Lemmas}
\crefname{thm}{theorem}{theorems}
\Crefname{thm}{Theorem}{Theorems}
\crefname{section}{Sec.}{sections}
\crefname{figure}{Fig.}{Fig.}
\Crefname{figure}{Figure}{Figures}
\crefformat{equation}{(#2#1#3)}
\crefrangeformat{equation}{(#3#1#4) to~(#5#2#6)}
\crefmultiformat{equation}{(#2#1#3)}%
{ and~(#2#1#3)}{, (#2#1#3)}{ and~(#2#1#3)}
\crefrangemultiformat{equation}{(#3#1#4) to~(#5#2#6)}%
{ and~(#3#1#4) to~(#5#2#6)}{, (#3#1#4) to~(#5#2#6)}{ and~(#3#1#4) to~(#5#2#6)}

\usepackage[detect-all]{siunitx}

\usepackage{array}
\newcolumntype{K}[1]{>{\centering\let\newline\\\arraybackslash\hspace{0pt}}m{#1}}

\usepackage[natbib=true,
    style=ieee,
    citestyle=numeric-comp,
    backend=biber,
    maxbibnames=10,
    maxcitenames=99,
    doi=false,
    url=false,
    giveninits=true]{biblatex}
\DeclareSourcemap{
  \maps[datatype=bibtex]{
    \map{
      \step[fieldsource=url,
        match=\regexp{\\_},
        replace=\regexp{_}]
    }
  }
}

\AtEveryBibitem{\clearfield{eprint}}
\AtEveryBibitem{\clearfield{issn}}

\AtBeginBibliography{\footnotesize}
\usepackage{tikz}
\usetikzlibrary{fit,calc}
\newcommand*{\tikzmk}[1]{\tikz[remember picture,overlay,] \node (#1) {};\ignorespaces}
\newcommand{\markline}[1]{\tikz[remember picture,overlay]{\node[yshift=3pt,xshift=#1,fill=blue!50,opacity=.25,fit={(A)($(A)+(0.95\linewidth,.3\baselineskip)$)},rounded corners=4pt] {};}\ignorespaces}

\newcommand{\B}[1]{\mathbf{#1}}
\newcommand{\BB}[1]{\mathbb{#1}}
\newcommand{\ns}{\textit{Neural-Swarm2}}
\newcommand{\fav}{\B{f}_a}

\newcommand{\tauav}{\bm{\tau}_a}

\newcommand{\famax}{f_{a,\mathrm{max}}}
\newcommand{\tauamax}{\tau_{a,\mathrm{max}}}
\newcommand{\favhat}{\hat{\B{f}}_a}
\newcommand{\tauavhat}{\hat{\bm{\tau}}_a}
\newcommand{\fnom}{\bm{\Phi}}
\newcommand{\set}{\mathbf{r}}
\newcommand{\env}{\mathrm{env}}
\newcommand{\sm}{\mathrm{small}}
\newcommand{\la}{\mathrm{large}}

\newcommand{\changed}[1]{#1}

\begin{document}
%
\title{Neural-Swarm2: Planning and Control of Heterogeneous Multirotor Swarms using\\ Learned Interactions}
%
%
%
\author{Guanya~Shi,
         Wolfgang~Hönig,
         Xichen~Shi,
         Yisong~Yue,
         and Soon-Jo~Chung
\thanks{The authors are with California Institute of Technology, USA.
\texttt{\{gshi, whoenig, xshi, yyue, sjchung\}@caltech.edu}.}
\thanks{The video is available at \url{https://youtu.be/Y02juH6BDxo}.}
\thanks{The work is funded in part by Caltech's Center for Autonomous Systems and Technologies (CAST) and the Raytheon Company.}}%

%
%

\markboth{IEEE Transactions on Robotics}%
{Shi \MakeLowercase{\textit{et al.}}: Neural-Swarm2}
%



\maketitle

\begin{abstract}
We present \ns{}, a learning-based method for motion planning and control that allows heterogeneous multirotors in a swarm to safely fly in close proximity. Such operation for drones is challenging due to complex aerodynamic interaction forces, such as downwash generated by nearby drones and ground effect. Conventional planning and control methods neglect capturing these interaction forces, resulting in sparse swarm configuration during flight. Our approach combines a physics-based nominal dynamics model with learned Deep Neural Networks (DNNs) with strong Lipschitz properties. We make use of two techniques to accurately predict the aerodynamic interactions between heterogeneous multirotors: i) spectral normalization for stability and generalization guarantees of unseen data and ii) heterogeneous deep sets for supporting any number of heterogeneous neighbors in a permutation-invariant manner without reducing expressiveness. The learned residual dynamics benefit both the proposed interaction-aware multi-robot motion planning and the nonlinear tracking control design because the learned interaction forces reduce the modelling errors. Experimental results demonstrate that \ns{} is able to generalize to larger swarms beyond training cases and significantly outperforms a baseline nonlinear tracking controller with up to three times reduction in worst-case tracking errors.
\end{abstract}


%
\IEEEpeerreviewmaketitle

\section{Introduction}

\IEEEPARstart{T}{he} ongoing commoditization of unmanned aerial vehicles (UAVs) requires robots to fly in much closer proximity to each other than before, which necessitates advanced planning and control methods for large aerial swarms~\cite{swarmsurvey,morgan2016swarm}.
For example, consider a search-and-rescue mission where an aerial swarm must enter and search a collapsed building.
In such scenarios, close-proximity flight enables the swarm to navigate the building much faster compared to swarms that must maintain large distances from each other.
Other important applications of close-proximity flight include manipulation, search, surveillance, and mapping.
In many scenarios, heterogeneous teams with robots of different sizes and sensing or manipulation capabilities are beneficial due to their significantly higher adaptability. For example, in a search-and-rescue mission, larger UAVs can be used for manipulation tasks or to transport goods, while smaller ones are more suited for exploration and navigation.

\begin{figure}[!h]
\includegraphics[width=\linewidth]{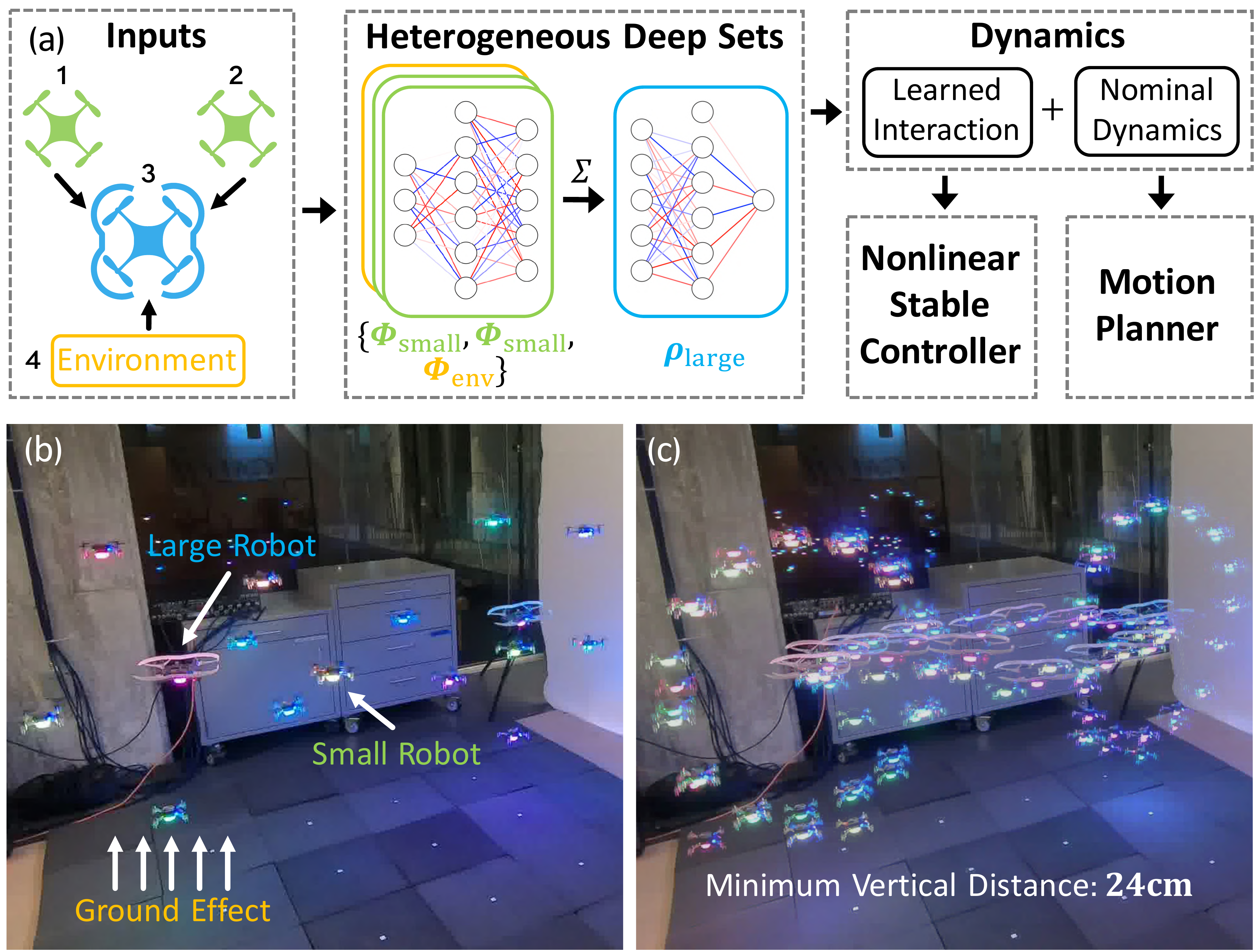}
\caption{We learn complex interaction between multirotors using heterogeneous deep sets and design an interaction-aware nonlinear stable controller and a multi-robot motion planner (a). Our approach enables close-proximity flight (minimum vertical distance \SI{24}{cm}) of heterogeneous aerial teams (16 robots) with significant lower tracking error compared to solutions that do not consider the interaction forces (b,c).}
\label{fig:fig1}
\end{figure}

A major challenge of close-proximity control and planning is that small distances between UAVs create complex aerodynamic interactions.
For instance, one multirotor flying above another causes the so-called downwash effect on the lower one, which is difficult to model using conventional model-based approaches~\cite{twoagent}.
Without accurate downwash interaction modeling, a large safety distance between vehicles is necessary, thereby preventing a compact 3-D formation shape, e.g., \SI{60}{cm} for the small Crazyflie 2.0 quadrotor (\SI{9}{cm} rotor-to-rotor)~\cite{DBLP:journals/trob/HonigPKSA18}. Moreover, the downwash is sometimes avoided by restricting the relative position between robots in the 2-D horizontal plane~\cite{DBLP:conf/icra/DuLVS19}.
For heterogeneous teams, even larger and asymmetric safety distances are required~\cite{DBLP:conf/iros/DebordHA18}.
However, the downwash for two small Crazyflie quadrotors hovering \SI{30}{cm} on top of each other is only \SI{9}{g}, which is well within their thrust capabilities, and suggests that proper modeling of downwash and other interaction effects can lead to more precise motion planning and dense formation control.

In this paper, we present a learning-based approach, \ns, which enhances the precision, safety, and density of close-proximity motion planning and control of heterogeneous multirotor swarms.
In the example shown in \cref{fig:fig1}, we safely operate the same drones with vertical distances less than half of those of prior work~\cite{DBLP:journals/trob/HonigPKSA18}.
In particular, we train deep neural networks (DNNs) to predict the residual interaction forces that are not captured by the nominal models of free-space aerodynamics. To the best of our knowledge, this is the first model for aerodynamic interactions between two or more multirotors in flight.
Our DNN architecture supports heterogeneous inputs in a permutation-invariant manner without reducing the expressiveness.
The DNN only requires relative positions and velocities of neighboring multirotors as inputs, similar to the existing collision-avoidance techniques~\cite{DBLP:conf/isrr/BergGLM09}, which enables fully-decentralized computation.
We use the predicted interaction forces to augment the nominal dynamics and derive novel methods to directly consider them during motion planning and as part of the multirotors' controller.

From a learning perspective, we leverage and extend two state-of-the-art tools to derive effective DNN models.
First, we extend deep sets~\cite{deepsets} to the heterogeneous case and prove its representation power.
Our novel encoding is used to model interactions between heterogeneous vehicle types in an index-free or permutation-invariant manner, enabling better generalization to new formations and a varying number of vehicles.
The second is spectral normalization~\cite{bartlett2017spectrally}, which ensures the DNN is Lipschitz continuous and helps the DNN generalize well on test examples that lie outside the training set.
We demonstrate that the interaction forces can be computationally efficiently and accurately learned such that a small 32-bit microcontroller can predict such forces in real-time.

From a planning and control perspective, we derive novel methods that directly consider the predicted interaction forces.
For motion planning, we use a two-stage approach. In the first stage, we extend an existing kinodynamic sampling-based planner for a single robot to the interaction-aware multi-robot case.
In the second stage, we adopt an optimization-based planner to refine the solutions of the first stage.
Empirically, we demonstrate that our interaction-aware motion planner both avoids dangerous robot configurations that would saturate the multirotors' motors and reduces the tracking error significantly.
For the nonlinear control we leverage the Lipschitz continuity of our learned interaction forces to derive stability guarantees similar to our prior work~\cite{shi2019neural-lander,shi2020neural}.
The controller can be used to reduce the tracking error of arbitrary desired trajectories, including ones that were not planned with an interaction-aware planner.

We validate our approach using two to sixteen quadrotors of two different sizes, and we also integrate ground effect and other unmodeled dynamics into our model, by viewing the physical environment as a special robot.
To our knowledge, our approach is the first that models interactions between two or more multirotor vehicles and demonstrates how to use such a model effectively and efficiently for motion planning and control of aerial teams.

\section{Related Work}
The aerodynamic interaction force applied to a single UAV flying near the ground (ground effect), has been modeled analytically~\cite{cheeseman1955effect, doi:10.2514/6.2015-1685,kumardownwash}.
In many cases, the ground effect is not considered in typical multirotor controllers and thus increases the tracking error of a multirotor when operating close to the ground.
However, it is possible to use ground effect prediction in real-time to reduce the tracking error~\cite{kumardownwash,shi2019neural-lander}.

The interaction between two rotor blades of a single multirotor has been studied in a lab setting to optimize the placement of rotors on the vehicle~\cite{drones2040043}.
However, it remains an open question how this influences the flight of two or more multirotors in close proximity.
Interactions between two multirotors can be estimated using a propeller velocity field model~\cite{twoagent}.
Unfortunately, this method is hard to generalize to the multi-robot or heterogeneous case and it only considers the stationary case, which is inaccurate for real flights.

The use of DNNs to learn higher-order residual dynamics or control actions is gaining attention in the areas of control and reinforcement learning settings \cite{shi2019neural-lander,o2021meta,le2016smooth,taylor2019episodic,cheng2019control,mckinnon2019learn,saveriano2017data,johannink2018residual}.
For swarms, a common encoding approach is to discretize the whole space and employ convolutional neural networks (CNNs), which yields a permutation-invariant encoding. Another common encoding for robot swarms is a Graphic Neural Network (GNN)~\cite{gnn, gnn-uav}. GNNs have been extended to heterogeneous graphs~\cite{heterogeneous-gnn}, but it remains an open research question how such a structure would apply to heterogeneous robot teams.
We extend a different architecture, which is less frequently used in robotics applications, called deep sets~\cite{deepsets}.
Deep sets enable distributed computation without communication requirements.
Compared to CNNs, our approach: i) requires less training data and computation; ii) is not restricted to a pre-determined resolution and input domain; and iii) directly supports the heterogeneous swarm.
Compared to GNNs, we do not require any direct communication between robots.
Deep sets have been used in robotics for homogeneous~\cite{shi2020neural} and heterogeneous~\cite{DBLP:journals/ral/RiviereHYC20} teams.
Compared to the latter~\cite{DBLP:journals/ral/RiviereHYC20}, our heterogeneous deep set extension has a more compact encoding and we prove its representation power.

For motion planning, empirical models have been used to avoid harmful interactions~\cite{morgan2016swarm,morgan2014model,DBLP:journals/trob/HonigPKSA18,DBLP:conf/iros/DebordHA18, DBLP:conf/icra/MellingerKK12}.
Typical safe boundaries along multi-vehicle motions form ellipsoids~\cite{DBLP:journals/trob/HonigPKSA18} or cylinders~\cite{DBLP:conf/iros/DebordHA18} along the motion trajectories.
Estimating such shapes experimentally would potentially lead to many collisions and dangerous flight tests and those collision-free regions are in general conservative.
In contrast, we use deep learning to estimate the interaction forces accurately in heterogeneous multi-robot teams.
This model allows us to directly control the magnitude of the interaction forces to accurately and explicitly control the risk, removing the necessity of conservative collision shapes.

\changed{We generalize and extend the results of our prior work~\cite{shi2020neural} as follows. i) We derive \emph{heterogeneous deep sets} to extend to the heterogeneous case and prove its expressiveness, which also unifies the approach with respect to single-agent residual dynamics learning (e.g., learning the ground effect for improved multirotor landing~\cite{shi2019neural-lander}) by regarding the environment as a special neighbor. ii) We present a novel two-stage method to use the learned interaction forces for multi-robot motion planning. iii) We explicitly compensate for the delay in motor speed commands in our position and attitude controllers, resulting in stronger experimental results for both our baseline and \ns{}. iv) Leveraging i)-iii), this paper presents a result of close-proximity flight with minimum vertical distance \SI{24}{cm} (the prior work~\cite{DBLP:journals/trob/HonigPKSA18} requires at least \SI{60}{cm} as the safe distance) of a 16-robot heterogeneous team in challenging tasks (e.g., the 3-ring task in \cref{fig:fig1}(b,c)). Our prior work~\cite{shi2020neural} only demonstrated 5-robot flights in relatively simple tasks.}

\section{Problem Statement}
\ns{} can generally apply to any robotic system and we will focus on multirotor UAVs in this paper.
We first present single multirotor dynamics including interaction forces modeled as disturbances.
Then, we generalize these dynamics for a swarm of multirotors.
Finally, we formulate our objective as a variant of an optimal control problem and introduce our performance metric.

\subsection{Single Multirotor Dynamics}
A single multirotor's state comprises of the global position $\B{p} \in \mathbb{R}^3$, global velocity $\B{v} \in \mathbb{R}^3$, attitude rotation matrix $\B{R} \in \mathrm{SO}(3)$, and body angular velocity $\bm{\omega} \in \mathbb{R}^3$.
Its dynamics are:
\begin{subequations}
\begin{align}
\dot{\B{p}} &= \B{v}, &  
m\dot{\B{v}} &=m\B{g}+\B{R}\B{f}_u + \fav,\label{eq:pos_dynamics} \\ 
\dot{\B{R}}&=\B{R}\B{S}(\bm{\omega}), & 
\B{J}\dot{\bm{\omega}} &= \B{J} \bm{\omega} \times \bm{\omega}  + \bm{\tau}_u + \tauav,
\label{eq:att_dynamics} \\
\bm{\eta} &= \B{B}_0\B{u}, & \dot{\B{u}} &= -\lambda \B{u} + \lambda \B{u}_c,
\label{eq:delay_model}
\end{align}
\label{eq:dynamics}
\end{subequations}
where $m$ and $\B{J}$ denote the mass and inertia matrix of the system, respectively; $\B{S}(\cdot)$ is a skew-symmetric mapping; $\B{g} = [0; 0; -g]$ is the gravity vector; and $\B{f}_u = [0; 0; T]$ and $\bm{\tau}_u = [\tau_x; \tau_y; \tau_z]$ denote the total thrust and body torques from the rotors, respectively. The output wrench $\bm{\eta}=[T;\tau_x;\tau_y;\tau_z]$ is linearly related to the control input $\bm{\eta}=\B{B}_0\B{u}$, where $\B{u}=[n_1^2;n_2^2;\ldots;n_M^2]$ is the squared motor speeds for a vehicle with $M$ rotors and $\B{B}_0$ is the actuation matrix.
A multirotor is subject to additional disturbance force $\fav=[f_{a,x}; f_{a,y}; f_{a,z}]$ and disturbance torque $\tauav=[\tau_{a,x}; \tau_{a,y}; \tau_{a,z}]$.
We also consider a first order delay model in \cref{eq:delay_model}, where $\B{u}_c$ is the actual command signal we can directly control, and $\lambda$ is the scalar time constant of the delay model. 

Our model creates additional challenges compared to other exisiting multirotor dynamics models (e.g.,~\cite{DBLP:conf/icra/MellingerKK12}). The first challenge stems from the effect of delay in \cref{eq:delay_model}. The second challenge stems from disturbance forces $\fav$ in \cref{eq:pos_dynamics} and disturbance torques $\tauav$ in \cref{eq:att_dynamics}, generated by the interaction between other multirotors and the environment.

\subsection{Heterogeneous Swarm Dynamics}
We now consider $N$ multirotor robots. We use $\B{x}^{(i)}=[\B{p}^{(i)};\B{v}^{(i)};\B{R}^{(i)};\bm{\omega}^{(i)}]$ to denote the state of the $i^{\mathrm{th}}$ multirotor. We use $\B{x}^{(ij)}$ to denote the \emph{relative} state component between robot $i$ and $j$, e.g., $\B{x}^{(ij)}=[\B{p}^{(j)}-\B{p}^{(i)};\B{v}^{(j)}-\B{v}^{(i)};\B{R}^{(i)}{\B{R}^{(j)}}^\top]$.

We use $\mathcal{I}(i)$ to denote the type of the $i^{\mathrm{th}}$ robot, where robots with identical physical parameters such as $m$, $\B{J}$, and $\B{B}_0$ are considered to be of the same type. We assume there are $K\leq N$ types of robots, i.e., $\mathcal{I}(\cdot)$ is a surjective mapping from $\{1,\cdots,N\}$ to $\{\mathrm{type}_1,\cdots,\mathrm{type}_K\}$. 
Let $\set^{(i)}_{\mathrm{type}_k}$ be the set of the relative states of the type$_k$ neighbors of robot $i$: 
\begin{equation}
\set^{(i)}_{\mathrm{type}_k} = \{\B{x}^{(ij)}\, |\, j\in\mathrm{neighbor}(i)\text{ and }\mathcal{I}(j)=\mathrm{type}_k \}.
\label{eq:rel-states}
\end{equation}
\changed{The neighboring function $\mathrm{neighbor}(i)$ is defined by an interaction volume function $\mathcal{V}$. Formally, $j\in\mathrm{neighbor}(i)$ if $\B{p}^{(j)}\in\mathcal{V}(\B{p}^{(i)},\mathcal{I}(i),\mathcal{I}(j))$, i.e., robot $j$ is a neighbor of $i$ if the position of $j$ is within the interaction volume of $i$. In this paper, we design $\mathcal{V}$ as a cuboid based on observed interactions in experiments.} The ordered sequence of all relative states grouped by robot type is
\begin{equation}
\set^{(i)}_{\mathcal{I}} = \left( \set^{(i)}_{\mathrm{type}_1}, \set^{(i)}_{\mathrm{type}_2}, \cdots, \set^{(i)}_{\mathrm{type}_K}\right).
\label{eq:all-rel-states}
\end{equation}

The dynamics of the $i^{\mathrm{th}}$ multirotor can be written in compact form:
\begin{equation}
\dot{\B{x}}^{(i)} = \fnom^{(i)}(\B{x}^{(i)},\B{u}^{(i)}) \\ 
+ \begin{bmatrix}
\B{0} \\ 
\fav^{(i)}(\set^{(i)}_{\mathcal{I}}) \\ 
\B{0} \\ 
\tauav^{(i)}(\set^{(i)}_{\mathcal{I}})
\end{bmatrix},
\label{eq:hetero-onerobot}        
\end{equation}
where $\fnom^{(i)}(\B{x}^{(i)},\B{u}^{(i)})$ denotes the nominal dynamics of robot $i$, and $\fav^{(i)}(\cdot)$ and $\tauav^{(i)}(\cdot)$ are the unmodeled force and torque of the $i^{\mathrm{th}}$ robot that are caused by interactions with neighboring robots or the environment (e.g., ground effect and air drag). 

Robots with the same type have the same nominal dynamics and unmodeled force and torque:
\begin{equation}
\fnom^{(i)}(\cdot)=\fnom^{\mathcal{I}(i)}(\cdot),\fav^{(i)}(\cdot)=\fav^{\mathcal{I}(i)}(\cdot),\tauav^{(i)}(\cdot)=\tauav^{\mathcal{I}(i)}(\cdot)\;\forall i.
\label{eq:hetero-interactions}
\end{equation}
Note that the homogeneous case covered in our prior work~\cite{shi2020neural} is a special case where $K=1$, i.e., $\fnom^{(i)}(\cdot)=\fnom(\cdot)$, $\fav^{(i)}(\cdot)=\fav(\cdot)$, and $\tauav^{(i)}(\cdot)=\tauav(\cdot)\;\forall i$.

Our system is heterogeneous in three ways: i) different robot types have heterogeneous nominal dynamics $\fnom^{\mathcal{I}(i)}$; ii) different robot types have different unmodeled $\fav^{\mathcal{I}(i)}$ and $\tauav^{\mathcal{I}(i)}$; and iii) the neighbors of each robot belong to $K$ different sets. 

We highlight that our heterogeneous model not only captures different types of robot, but also different types of environmental interactions, e.g., ground effect~\cite{shi2019neural-lander} and air drag. This is achieved in a straightforward manner by viewing the physical environment as a special robot type. We illustrate this generalization in the following example.

\begin{example}[small and large robots, and the environment]
\label{example:hetero-system}
We consider a heterogeneous system as depicted in \cref{fig:fig1}(a). Robot $3$ (large robot) has three neighbors: robot 1 (small), robot 2 (small) and environment 4. For robot $3$, we have
\begin{equation*}
\begin{aligned}
\B{f}_{a}^{(3)} &= \fav^{\la}(\set^{(3)}_{\mathcal{I}}) = 
\fav^{\la}(\set^{(3)}_{\sm}, \set^{(3)}_{\la}, \set^{(3)}_{\env}),\\
\set_{\sm}^{(3)} &= \{\B{x}^{(31)}, \B{x}^{(32)}\},\,\,\, \set_{\la}^{(3)} = \emptyset,\,\,\,
\set_{\env}^{(3)} = \{\B{x}^{(34)}\}
\end{aligned}    
\end{equation*}
and a similar expression for $\bm{\tau}_{a}^{(3)}$.
\end{example}

\subsection{Interaction-Aware Motion Planning \& Control}
Our goal is to move the heterogeneous team of robots from their start states to goal states, which can be framed as the following optimal control problem:
\begin{align}
&\min_{\B{u}^{(i)},\B{x}^{(i)}, t_f} \sum_{i=1}^N \int_{0}^{t_f} \| \B{u}^{(i)}(t) \| dt \label{eq:motion-planning}
\\
&\text{\noindent s.t.}\begin{cases}
\text{\noindent robot dynamics } \eqref{eq:hetero-onerobot} &i\in[1,N] \\ 
\B{u}^{(i)}(t) \in \mathcal{U}^{\mathcal{I}(i)};\,\, \B{x}^{(i)}(t) \in \mathcal{X}^{\mathcal{I}(i)}  &i\in[1,N] \\ 
\|\B{p}^{(ij)}\| \geq r^{(\mathcal{I}(i) \mathcal{I}(j))} &i < j,\; j\in[2,N] \\ 
\|\B{f}_{a}^{(i)}\| \leq \famax^{\mathcal{I}(i)};\,\, \|\bm{\tau}_{a}^{(i)}\| \leq \tauamax^{\mathcal{I}(i)} &i\in[1,N] \\ 
\B{x}^{(i)}(0) = \B{x}_s^{(i)};\,\,\B{x}^{(i)}(t_f) = \B{x}_f^{(i)} &i\in[1,N]
\end{cases} \nonumber
\end{align}
where $\mathcal{U}^{(k)}$ is the control space for $\mathrm{type}_k$ robots, $\mathcal{X}^{(k)}$ is the free space for $\mathrm{type}_k$ robots, $r^{(lk)}$ is the minimum safety distance between $\mathrm{type}_l$ and $\mathrm{type}_k$ robots, $\famax^{(k)}$ is the maximum endurable interaction force for $\mathrm{type}_k$ robots, $\tauamax^{(k)}$ is the maximum endurable interaction torque for $\mathrm{type}_k$ robots, $\B{x}_s^{(i)}$ is the start state of robot $i$, and $\B{x}_f^{(i)}$ is the desired state of robot $i$.
In contrast with the existing literature~\cite{DBLP:conf/iros/DebordHA18}, we assume a tight spherical collision model and bound the interaction forces directly, eliminating the need of manually defining virtual collision shapes. For instance, larger $\famax^{\mathcal{I}(i)}$ and $\tauamax^{\mathcal{I}(i)}$ will yield denser and more aggressive trajectories. 
Also note that the time horizon $t_f$ is a decision variable.

Solving \cref{eq:motion-planning} in real-time in a distributed fashion is intractable due to the exponential growth of the decision space with respect to the number of robots.
Thus, we focus on solving two common subproblems instead.
First, we approximately solve \cref{eq:motion-planning} offline as an interaction-aware motion planning problem.
Second, we formulate an interaction-aware controller that minimizes the tracking error online. This controller can use both predefined trajectories and planned trajectories from the interaction-aware motion planner.

Since interaction between robots might only occur for a short time period with respect to the overall flight duration but can cause significant deviation from the nominal trajectory, we consider the worst tracking error of any robot in the team as a success metric:
\begin{equation}
    \max_{i,t} \|\B{p}^{(i)}(t)-\B{p}_d^{(i)}(t)\|,
    \label{eq:metric}
\end{equation}
where $\B{p}_d^{(i)}(t)$ is the desired trajectory for robot $i$. \changed{Note that this metric reflects the worst error out of \emph{all} robots, because different robots in a team have various levels of task difficulty. For example, the two-drone swapping task in \cref{fig:plot4} is very challenging for the bottom drone due to the downwash effect, but relatively easier for the top drone.}
Minimizing \cref{eq:metric} implies improved tracking performance and safety of a multirotor swarm during tight formation flight.

\section{Learning of Swarm Aerodynamic Interaction}
We employ state-of-the-art deep learning methods to capture the unknown (or residual) dynamics caused by interactions of heterogeneous robot teams. In order to use the learned functions effectively for motion planning and control, we require that the DNNs have strong Lipschitz properties (for stability analysis), can generalize well to new test cases, and use compact encodings to achieve high computational and statistical efficiency. To that end, we introduce \emph{heterogeneous deep sets}, a generalization of regular deep sets~\cite{deepsets}, and employ spectral normalization~\cite{bartlett2017spectrally} for strong Lipschitz properties.

In this section, we will first review the homogeneous learning architecture covered in prior work~\cite{deepsets,shi2020neural}. Then we will generalize them to the heterogeneous case with representation guarantees. Finally, we will introduce spectral normalization and our data collection procedures.  

\subsection{Homogeneous Permutation-Invariant Neural Networks}
Recall that in the homogeneous case, all robots are with the same type ($\mathrm{type}_1$). Therefore, the input to functions $\fav$ or $\tauav$ is a single set. The permutation-invariant aspect of $\fav$ or $\tauav$ can be characterized as:
\begin{equation*}
\fav(\set_{\mathrm{type}_1}^{(i)}) = \fav(\pi(\set_{\mathrm{type}_1}^{(i)})),\,\tauav(\set_{\mathrm{type}_1}^{(i)}) = \tauav(\pi(\set_{\mathrm{type}_1}^{(i)}))
\end{equation*}
for any permutation $\pi$.
Since the aim is to learn the function $\fav$ and $\tauav$ using DNNs, we need to guarantee that the learned DNN is permutation-invariant. Therefore,
we consider the following ``deep sets''~\cite{deepsets} architecture to approximate homogeneous $\fav$ and $\tauav$:
\begin{equation}
\begin{bmatrix}
\fav(\set_{\mathrm{type}_1}^{(i)}) \\
\tauav(\set_{\mathrm{type}_1}^{(i)})
\end{bmatrix}
\approx
\bm{\rho} \left(\sum_{\B{x}^{(ij)}\in \set_{\mathrm{type}_1}^{(i)}}\bm{\phi}(\B{x}^{(ij)})\right)\coloneqq
\begin{bmatrix}
\favhat^{(i)} \\
\tauavhat^{(i)}
\end{bmatrix}
,
\label{eq:learningmodel}
\end{equation}
where $\bm{\phi}(\cdot)$ and $\bm{\rho}(\cdot)$ are two DNNs. The output of $\bm{\phi}$ is a hidden state to represent ``contributions'' from each neighbor, and $\bm{\rho}$ is a nonlinear mapping from the summation of these hidden states to the total effect. 

Obviously the network architecture in \cref{eq:learningmodel} is permutation-invariant due to the inner sum operation. We now show that this architecture is able to approximate any continuous permutation-invariant function.
The following Lemma~\ref{lemma:homeomorphism} and \cref{thm:deepsets} are adopted from~\cite{deepsets} and will be used and extended in the next section for the heterogeneous case.

\begin{lemma}\label{lemma:homeomorphism}
Define $\bar{\bm{\phi}}(z)=[1;z;\cdots;z^M]\in\BB{R}^{M+1}$ as a mapping from $\BB{R}$ to $\BB{R}^{M+1}$, and $\mathcal{X}=\{[x_1;\cdots;x_M]\in[0,1]^M | x_1\leq\cdots\leq x_M\}$ as a subset of $[0,1]^M$. For $\B{x}=[x_1;\cdots;x_M]\in\mathcal{X}$, define $\B{q}(\B{x})=\sum_{m=1}^M \bar{\bm{\phi}}(x_m)$. Then $\B{q}(\B{x}):\mathcal{X} \to \BB{R}^{M+1}$ is a homeomorphism. 
\end{lemma}
\begin{proof}
The proof builds on the Newton-Girard formulae, which connect the moments of a sample set (sum-of-power) to the elementary symmetric polynomials (see~\cite{deepsets}).
\end{proof}

\begin{theorem}\label{thm:deepsets}
Suppose $h(\B{x}):[0,1]^M\rightarrow\BB{R}$ is a permutation-invariant continuous function, i.e., $h(\B{x})=h(x_1,\cdots,x_M)=h(\pi(x_1,\cdots,x_M))$ for any permutation $\pi$. Then there exist continuous functions $\bar{\rho}:\BB{R}^{M+1}\rightarrow\BB{R}$ and $\bar{\bm{\phi}}:\BB{R}\rightarrow\BB{R}^{M+1}$ such that 
\begin{equation*}
h(\B{x}) = \bar{\rho}\left( \sum_{m=1}^M \bar{\bm{\phi}}(x_m) \right), \quad \forall \B{x}\in[0,1]^M.
\end{equation*}
\end{theorem}
\begin{proof}
We choose $\bar{\bm{\phi}}(z)=[1;z;\cdots;z^M]$ and $\bar{\rho}(\cdot)=h(\B{q}^{-1}(\cdot))$, where $\B{q}(\cdot)$ is defined in Lemma \ref{lemma:homeomorphism}.
Note that since $\B{q}(\cdot)$ is a homeomorphism, $\B{q}^{-1}(\cdot)$ exists and it is a continuous function from $\BB{R}^{M+1}$ to $\mathcal{X}$. Therefore, $\bar{\rho}$ is also a continuous function from $\BB{R}^{M+1}$ to $\BB{R}$, and $\bar{\rho}\left( \sum_{m=1}^M \bar{\bm{\phi}}(x_m) \right)=\bar{\rho}(\B{q}(\B{x}))=h(\B{q}^{-1}(\B{q}(\B{x})))=h(\B{x})$ for $\B{x}\in\mathcal{X}$. Finally, note that for any $\B{x}\in[0,1]^M$, there exists some permutation $\pi$ such that $\pi(\B{x})\in\mathcal{X}$. 
Then because both $\bar{\rho}(\B{q}(\B{x}))$ and $h(\B{x})$ are permutation-invariant, we have $\bar{\rho}\left( \B{q}(\B{x}) \right) = \bar{\rho}\left( \B{q}(\pi(\B{x})) \right) = h(\pi(\B{x})) = h(\B{x})$ for all $\B{x}\in[0,1]^M$.
\end{proof}

Theorem~\ref{thm:deepsets} focuses on scalar valued permutation-invariant continuous functions with hypercubic input space $[0,1]^M$, i.e., each element in the input set is a scalar. In contrast, our learning target function $[\fav;\tauav]$ in \cref{eq:learningmodel} is a vector valued function with a bounded input space, and each element in the input set is also a vector. However, Theorem~\ref{thm:deepsets} can be generalized in a straightforward manner by the following corollary.

\begin{corollary}
\label{corollary:homo-vector}
Suppose $\B{x}^{(1)},\B{x}^{(2)},\cdots,\B{x}^{(M)}$ are $M$ bounded vectors in $\BB{R}^{D_1}$, and $h(\B{x}^{(1)},\cdots,\B{x}^{(M)})$ is a continuous permutation-invariant function from $\BB{R}^{M\times D_1}$ to $\BB{R}^{D_2}$, i.e., $h(\B{x}^{(1)},\cdots,\B{x}^{(M)})=h(\B{x}^{\pi(1)},\cdots,\B{x}^{\pi(M)})$ for any permutation $\pi$. Then $h(\B{x}^{(1)},\cdots,\B{x}^{(M)})$ can be approximated arbitrarily close in the proposed architecture in \cref{eq:learningmodel}.
\end{corollary}
\begin{proof}
First, there exists a bijection from the bounded vector space in $\BB{R}^{D_1}$ to $[0,1]$ after discretization, with finite but arbitrary precision. Thus, Theorem~\ref{thm:deepsets} is applicable.
Second, we apply Theorem~\ref{thm:deepsets} $D_2$ times and stack $D_2$ scalar-valued functions to represent the vector-valued function with output space $\BB{R}^{D_2}$.
Finally, because DNNs are universal approximators for continuous functions~\cite{csaji2001approximation}, the proposed architecture in~\cref{eq:learningmodel} can approximate any $h(\B{x}^{(1)},\cdots,\B{x}^{(M)})$ arbitrarily close.
\end{proof}
 
\subsection{Heterogeneous $K$-Group Permutation-Invariant DNN}
Different from the homogeneous setting, the inputs to functions $\fav^{\mathcal{I}(i)}$ and $\tauav^{\mathcal{I}(i)}$ in \cref{eq:hetero-interactions} are $K$ different sets. First, we define \emph{permutation-invariance} in the heterogeneous case. Intuitively, we expect that the following equality holds:
\begin{equation*}
\fav^{\mathcal{I}(i)}(\set_{\mathrm{type}_1}^{(i)},\cdots,\set_{\mathrm{type}_K}^{(i)}\!)  = \fav^{\mathcal{I}(i)}(\pi_1(\set_{\mathrm{type}_1}^{(i)}\!),\cdots,\pi_K(\set_{\mathrm{type}_K}^{(i)\!}))    
\end{equation*}
for any permutations $\pi_1,\cdots,\pi_K$ (similarly for $\tauav^{\mathcal{I}(i)}$). Formally, we define $K$-group permutation invariance as follows. 
\begin{definition}[$K$-group permutation invariance]
Let $\B{x}^{(k)}=[x^{(k)}_1;\cdots;x^{(k)}_{M_k}]\in[0,1]^{M_k}$ for $1\leq k \leq K$, and $\B{x}=[\B{x}^{(1)};\cdots;\B{x}^{(K)}]\in[0,1]^{M_K}$, where $M_K=\sum_{k=1}^K M_k$. $h(\B{x}):\BB{R}^{M_K}\rightarrow\BB{R}$ is $K$-group permutation-invariant if
\begin{equation*}
h([\B{x}^{(1)};\cdots;\B{x}^{(K)}]) = h([\pi_1(\B{x}^{(1)});\cdots;\pi_K(\B{x}^{(K)})])
\end{equation*}
for any permutations $\pi_1,\pi_2,\cdots,\pi_K$.
\end{definition}

For example, $h(x_1,x_2,y_1,y_2)=\max\{x_1,x_2\}+2\cdot\max\{y_1,y_2\}$ is a 2-group permutation-invariant function, because we can swap $x_1$ and $x_2$ or swap $y_1$ and $y_2$, but if we interchange $x_1$ and $y_1$ the function value may vary. In addition, the $\fav^{\la}$ function in Example~\ref{example:hetero-system} is a 3-group permutation-invariant function.

Similar to Lemma \ref{lemma:homeomorphism}, in order to handle ambiguity due to permutation, we define $\mathcal{X}_{M_k}=\{[x_1;\cdots;x_{M_k}]\in[0,1]^{M_k} | x_1\leq\cdots\leq x_{M_k}\}$ and
\begin{equation*}
\mathcal{X}_K=\{[\B{x}^{(1)};\cdots;\B{x}^{(K)}] \in[0,1]^{M_K} | \B{x}^{(k)} \in \mathcal{X}_{M_k},\forall k\}.
\end{equation*}
Finally, we show how a $K$-group permutation-invariant function can be approximated \changed{via} the following theorem.

\begin{figure}
\includegraphics[width=\linewidth]{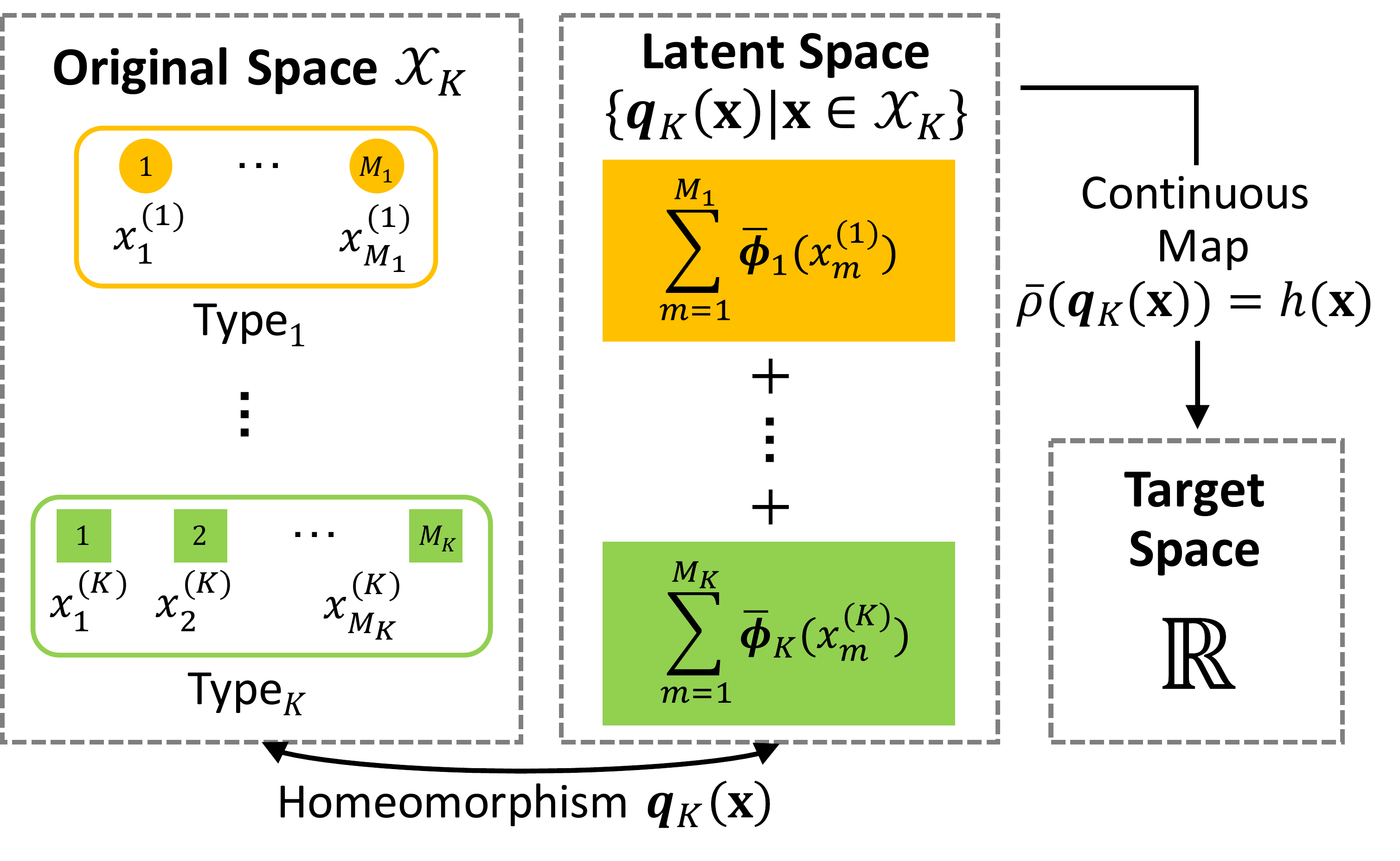}
\caption{Illustration of \Cref{thm:deepsets-hete}. We first find a homeomorphism $\B{q}_K(\cdot)$ between the original space and the latent space, and then find a continuous function $\bar{\rho}(\cdot)$ such that $\bar{\rho}(\B{q}_K(\cdot))=h(\cdot)$.}
\label{fig:hetero_illustration}
\end{figure}

\begin{theorem}\label{thm:deepsets-hete}
$h(\B{x}):[0,1]^{M_K}\rightarrow\BB{R}$ is a $K$-group permutation-invariant continuous function if and only if it has the representation
\begin{equation*}
\begin{aligned}
h(\B{x}) &= \bar{\rho}\left( \sum_{m=1}^{M_1} \bar{\bm{\phi}}_1(x_m^{(1)}) + \cdots + \sum_{m=1}^{M_K} \bar{\bm{\phi}}_K(x_m^{(K)})\right) \\ &= \bar{\rho}\left( \sum_{k=1}^K\sum_{m=1}^{M_k} \bar{\bm{\phi}}_k(x_m^{(k)}) \right), \quad \forall \B{x}\in[0,1]^{M_K}
\end{aligned}
\end{equation*}
for some continuous outer and inner functions $\bar{\rho}:\BB{R}^{K+M_K}\rightarrow\BB{R}$ and $\bar{\bm{\phi}}_k:\BB{R}\rightarrow\BB{R}^{K+M_K}$ for $1\leq k \leq K$. 
\end{theorem}
\begin{proof}
The sufficiency follows from that $h(\B{x})$ is $K$-group permutation-invariant by construction. For the necessary condition, we need to find continuous functions $\bar{\rho}$ and $\{\bar{\bm{\phi}}_k\}_{k=1}^K$ given $h$.
We define $\bar{\bm{\phi}}_k(x):\BB{R}\rightarrow\BB{R}^{K+M_K}$ as
\begin{equation*}
\bar{\bm{\phi}}_k(x) = [\B{0}_{M_1};\cdots;\B{0}_{M_{k-1}};
\begin{bmatrix}
1 \\
x \\
\vdots \\
x^{M_k}
\end{bmatrix};
\B{0}_{M_{k+1}};\cdots;\B{0}_{M_K}] 
\end{equation*}
where $\B{0}_{M_k}=[0;\cdots;0]\in\BB{R}^{M_k+1}$. 
Then 
\begin{equation*}
\B{q}_K(\B{x}) = \sum_{k=1}^K\sum_{m=1}^{M_k} \bar{\bm{\phi}}_k(x_m^{(k)})    
\end{equation*}
is a homeomorphism from $\mathcal{X}_K\subseteq\BB{R}^{M_K}$ to $\BB{R}^{K+M_K}$ from Lemma \ref{lemma:homeomorphism}. We choose $\bar{\rho}:\BB{R}^{K+M_K}\rightarrow\BB{R}$ as $\bar{\rho}(\cdot)=h(\B{q}_K^{-1}(\cdot))$ which is continuous, because both $\B{q}_K^{-1}$ and $h$ are continuous. Then $\bar{\rho}(\B{q}_K(\B{x}))=h(\B{x})$ for $\B{x}\in\mathcal{X}_K$. Finally, because i) for all $\B{x}=[\B{x}^{(1)};\cdots;\B{x}^{(K)}]$ in $[0,1]^{M_K}$ there exist permutations $\pi_1,\cdots,\pi_K$ such that $[\pi_1(\B{x}^{(1)});\cdots;\pi_K(\B{x}^{(K)})]\in\mathcal{X}_K$; and ii) both $\bar{\rho}(\B{q}_K(\B{x}))$ and $h(\B{x})$ are $K$-group permutation-invariant, we have $\bar{\rho}(\B{q}_K(\B{x}))=h(\B{x})$ for $\B{x}\in[0,1]^{M_K}$.
\end{proof}

\Cref{fig:hetero_illustration} depicts the key idea of \Cref{thm:deepsets-hete}. Moreover, we provide a 2-group permutation-invariant function example to highlight the roles of $\bm{\phi}$ and $\rho$ in the heterogeneous case.
\begin{example}[2-group permutation-invariant function]
Consider $h(x_1,x_2,y_1,y_2)=\max\{x_1,x_2\} + 2\cdot\max\{y_1,y_2\}$, which is $2$-group permutation-invariant. Then we define $\bm{\phi}_x(x)=[e^{\alpha x};xe^{\alpha x};0;0]$, $\bm{\phi}_y(y)=[0;0;e^{\alpha y};ye^{\alpha y}]$ and $\rho([a;b;c;d]) = b/a + 2\cdot d/c$. Note that
\begin{equation*}
\begin{aligned}
&\rho(\bm{\phi}_x(x_1)+\bm{\phi}_x(x_2)+\bm{\phi}_y(y_1)+\bm{\phi}_y(y_2))\\=&
\frac{x_1e^{\alpha x_1} + x_2e^{\alpha x_2}}{e^{\alpha x_1} + e^{\alpha x_2}} +
2\cdot\frac{y_1e^{\alpha y_1} + y_2e^{\alpha y_2}}{e^{\alpha y_1} + e^{\alpha y_2}},
\end{aligned}
\end{equation*}
which is asymptotically equal to $\max\{x_1,x_2\} + 2\cdot\max\{y_1,y_2\}$ as $\alpha\rightarrow+\infty$.
\end{example}

Similar to the homogeneous case, Theorem~\ref{thm:deepsets-hete} can generalize to vector-output functions with a bounded input space by applying the same argument as in Corollary~\ref{corollary:homo-vector}. We propose the following \emph{heterogeneous deep set} structure to model the heterogeneous functions $\B{f}_a^{\mathcal{I}(i)}$ and $\bm{\tau}_a^{\mathcal{I}(i)}$:
\begin{equation}
\begin{aligned}
&\begin{bmatrix}
\fav^{\mathcal{I}(i)}(\set_{\mathrm{type}_1}^{(i)},\cdots,\set_{\mathrm{type}_K}^{(i)}) \\
\tauav^{\mathcal{I}(i)}(\set_{\mathrm{type}_1}^{(i)},\cdots,\set_{\mathrm{type}_K}^{(i)})
\end{bmatrix}
\\
\approx&
\bm{\rho}_{\mathcal{I}(i)}\left(\sum_{k=1}^K\sum_{\B{x}^{(ij)}\in \set_{\mathrm{type}_k}^{(i)}}\bm{\phi}_{\mathcal{I}(j)}(\B{x}^{(ij)})\right)\coloneqq
\begin{bmatrix}
\favhat^{(i)} \\ \tauavhat^{(i)}
\end{bmatrix}.
\end{aligned}
\label{eq:learningmodel-hetero} 
\end{equation}

\begin{example}[Use of 3-group permutation-invariant function for multirotors]
For example, in the heterogeneous system provided by Example~\ref{example:hetero-system} (as depicted in \cref{fig:fig1}(a)), we have
\begin{equation*}
\begin{aligned}
&\begin{bmatrix}
\fav^{(3)} \\ \tauav^{(3)}
\end{bmatrix} = 
\begin{bmatrix}
\fav^{\la}(\set_{\sm}^{(3)},\set_{\la}^{(3)},\set_{\env}^{(3)}) \\
\tauav^{\la}(\set_{\sm}^{(3)},\set_{\la}^{(3)},\set_{\env}^{(3)})
\end{bmatrix}
\\
&\approx \bm{\rho}_{\la}\left(\bm{\phi}_{\sm}(\B{x}^{(31)})+\bm{\phi}_{\sm}(\B{x}^{(32)})+\bm{\phi}_{\env}(\B{x}^{(34)})\right),
\end{aligned}    
\end{equation*}
for the large robot 3, where $\bm{\phi}_{\sm}$ captures the interaction with the small robot 1 and 2, and $\bm{\phi}_{\env}$ captures the interaction with the environment 4, e.g., ground effect and air drag.
\end{example}

The structure in \cref{eq:learningmodel-hetero} has many valuable properties:
\begin{itemize}
\item \textbf{Representation ability.} Since Theorem~\ref{thm:deepsets-hete} is necessary and sufficient, we do not lose approximation power by using this constrained framework, i.e., any $K$-group permutation-invariant function can be learned by \cref{eq:learningmodel-hetero}.  We demonstrate strong empirical performance using relatively compact DNNs for $\bm{\rho}_{\mathcal{I}(i)}$ and $\bm{\phi}_{\mathcal{I}(j)}$.

\item \textbf{Computational and sampling efficiency and scalability.} Since the input dimension of $\bm{\phi}_{\mathcal{I}(j)}$ is always the same as the single vehicle case, the feed-forward computational complexity of $\cref{eq:learningmodel-hetero}$ grows linearly with the number of neighboring vehicles. Moreover, the number of neural networks ($\bm{\rho}_{\mathcal{I}(i)}$ and $\bm{\phi}_{\mathcal{I}(j)}$) we need is $2K$, which grows linearly \changed{with} the number of robot types. In practice, we found that one hour flight data is sufficient to accurately learn interactions between two to five multirotors.

\item \textbf{Generalization to varying swarm size.} \changed{Given learned $\bm{\phi}_{\mathcal{I}(j)}$ and $\bm{\rho}_{\mathcal{I}(i)}$ functions, \cref{eq:learningmodel-hetero} can be used to predict interactions for any swarm size. In other words, we can accurately model swarm sizes (slightly) larger than those used for training. In practice, we found that our model can give good predictions for five multirotor swarms, despite only being trained on one to three multirotor swarms. Theoretical analysis on this generalizability is an interesting future research direction.}
\end{itemize}

\subsection{Spectral Normalization for Robustness and Generalization}
\label{sec:spectral}
To improve a property of robustness and generalizability of DNNs, we use spectral normalization~\cite{bartlett2017spectrally} for training optimization. Spectral normalization stabilizes a DNN training by constraining its Lipschitz constant. Spectrally\changed{-}normalized DNNs have been shown to generalize well, which is an indication of stability in machine learning.  Spectrally\changed{-}normalized DNNs have also been shown to be robust, which can be used to provide control-theoretic stability guarantees~\cite{liu2020robust,shi2019neural-lander}. The bounded approximation error assumption (Assumption \ref{assump:bounded_error}) in our control stability and robustness analysis (Sec.~\ref{sec:stability}) also relies on spectral normalization of DNNs.  

Mathematically, the Lipschitz constant $\|\B{g}\|_{\mathrm{Lip}}$ of a function $\B{g}(\cdot)$ is defined as the smallest value such that:
\[\forall \, \B{x}, \B{x}':\ \|\B{g}(\B{x})-\B{g}(\B{x}')\|_2/\|\B{x}-\B{x}'\|_2\leq \|\B{g}\|_{\mathrm{Lip}}.\]
Let $\B{g}(\B{x},\bm{\theta})$ denote a ReLU DNN parameterized by the DNN weights $\bm{\theta}={\B{W}_1,\cdots,\B{W}_{L+1}}$:
\begin{equation}
\B{g}(\B{x},\bm{\theta}) = \B{W}_{L+1}\sigma(\B{W}_L\sigma(\cdots \sigma(\B{W}_1\B{x})\cdots)),
\end{equation}
where the activation function $\sigma(\cdot)=\max(\cdot,0)$ is called the element-wise ReLU function.
In practice, we apply the spectral normalization to the weight matrices in each layer after each batch gradient descent as follows:
\begin{equation}
\B{W}_i \leftarrow \B{W}_i / \|\B{W}_i\|_2 \cdot \gamma^{\frac{1}{L+1}},i\in[1,L+1],
\label{eq:sn}
\end{equation}
where $\|\B{W}_i\|_2$ is the maximum singular value of $\B{W}_i$ and $\gamma$ is a hyperparameter. With \cref{eq:sn}, $\|\B{g}\|_{\mathrm{Lip}}$ will be upper bounded by $\gamma$. Since spectrally-normalized $\B{g}$ is $\gamma-$Lipschitz continuous, it is robust to noise $\Delta\B{x}$, i.e., $\|\B{g}(\B{x}+\Delta\B{x})-\B{g}(\B{x})\|_2$ is always bounded by $\gamma\|\Delta\B{x}\|_2$. In this paper, we apply the spectral normalization on both the $\bm{\phi}_{\mathcal{I}(j)}(\cdot)$ and $\bm{\rho}_{\mathcal{I}(i)}(\cdot)$ DNNs in \cref{eq:learningmodel-hetero}.

\subsection{Curriculum Learning}
Training DNNs in \cref{eq:learningmodel-hetero} to approximate $\fav^{\mathcal{I}(i)}$ and $\tauav^{\mathcal{I}(i)}$ requires collecting close formation flight data.
However, the downwash effect causes the nominally controlled multirotors (without compensation for the interaction forces) to move apart from each other.
Thus, we use a curriculum/cumulative learning approach: first, we collect data for two multirotors without a DNN and learn a model.
Second, we repeat the data collection using our learned model as a feed-forward term in our controller, which allows closer-proximity flight of the two vehicles.
Third, we repeat the procedure with increasing number of vehicles, using the current best model.

Note that our data collection and learning are independent of the controller used and independent of the $\fav^{\mathcal{I}(i)}$ or $\tauav^{\mathcal{I}(i)}$ compensation.
In particular, if we actively compensate for the learned $\fav^{\mathcal{I}(i)}$ or $\tauav^{\mathcal{I}(i)}$, this will only affect $\bm{\eta}$ in \eqref{eq:pos_dynamics} and not the observed $\fav^{\mathcal{I}(i)}$ or $\tauav^{\mathcal{I}(i)}$.

\section{Interaction-Aware Multi-Robot Planning}
\label{sec:planning}
We approximately solve \eqref{eq:motion-planning} offline by using two simplifications: i) we plan sequentially for each robot, treating other robots as dynamic obstacles with known trajectories, and ii) we use double-integrator dynamics plus learned interactions.
Both simplifications are common for multi-robot motion planning with applications to multirotors~\cite{morgan2016swarm,DBLP:journals/ral/LuisS19}.
Such a motion planning approach can be easily distributed and is complete for planning instances that fulfill the \emph{well-formed infrastructure} property~\cite{WellFormedInfrastructures}.
However, the interaction forces \eqref{eq:learningmodel-hetero} complicate motion planning significantly, because the interactions are highly nonlinear and robot dynamics are not independent from each other anymore.

For example, consider a three robot team with two small and one large robot as in \cref{fig:fig1}(a).
Assume that we already have valid trajectories for the two small robots and now plan a motion for the large robot.
The resulting trajectory might result in a significant downwash force for the small robots if the large robot flies directly above the small ones.
This strong interaction might invalidate the previous trajectories of the small robots or even violate their interaction force limits $\famax^{\sm}$ and $\tauamax^{\sm}$.
Furthermore, the interaction force is asymmetric and thus it is not sufficient to only consider the interaction force placed on the large robot.
We solve this challenge by directly limiting the change of the interaction forces placed on all neighbors when we plan for a robot.
This concept is similar to trust regions in sequential optimization~\cite{foustOptimalGuidanceControl2020}.

The simplified state is $\B{x}^{(i)}=[\B{p}^{(i)};\B{v}^{(i)};\favhat^{(i)}]$ and the simplified dynamics \eqref{eq:hetero-onerobot} become:
\begin{equation}
\dot{\B{x}}^{(i)} 
= \B{f}^{(i)}(\B{x}^{(i)}, \B{u}^{(i)}) = 
\begin{bmatrix}
\B{v}^{(i)} \\ \B{u}^{(i)} + \favhat^{(i)} \\ \dot{\hat{\B{f}}}_{a}^{(i)}
\end{bmatrix}.
\label{eq:planning-dynamics}
\end{equation}
These dynamics are still complex and nonlinear because of $\favhat^{(i)}$, which is the learned interaction force represented by DNNs in~\cref{eq:learningmodel-hetero}.
We include $\favhat^{(i)}$ in our state space to simplify the enforcement of the bound on the interaction force in~\eqref{eq:motion-planning}.

We propose a novel hybrid two-stage planning algorithm, see \cref{alg:planning}, leveraging the existing approaches while still highlighting the importance of considering interactive forces/torques in the planning. The portions of the pseudo-code in \cref{alg:planning} that significantly differ from the existing methods to our approach are highlighted.
In Stage 1, we find initial feasible trajectories using a kinodynamic sampling-based motion planner.
Note that any kinodynamic planner can be used for Stage 1. 
In Stage 2, we use sequential convex programming (SCP)~\cite{morgan2016swarm,foustOptimalGuidanceControl2020,SCP} to refine the initial solution to reach the desired states exactly and to minimize our energy objective defined in \eqref{eq:motion-planning}.
Intuitively, Stage 1 identifies the homeomorphism class of the solution trajectories and fixes $t_f$, while Stage 2 finds the optimal trajectories to the goal within that homeomorphism class.
Both stages differ from similar methods in the literature~\cite{morgan2016swarm}, because they need to reason over the coupling of the robots caused by interaction forces $\favhat^{(i)}$. 

\SetAlFnt{\small\sf}
\begin{algorithm}
\DontPrintSemicolon
\SetInd{0.25em}{0.5em}

\SetKwFunction{RandomShuffle}{RandomShuffle}
\SetKwFunction{FindClosest}{FindClosest}
\SetKwFunction{UniformSample}{UniformSample}
\SetKwFunction{Propagate}{Propagate}
\SetKwFunction{StateValid}{StateValid}
\SetKwFunction{Add}{Add}
\SetKwFunction{ExtractSolution}{ExtractSolution}
\SetKwFunction{PostProcess}{PostProcess}
\SetKwFunction{TerminationConditionOne}{TerminationCondition1}
\SetKwFunction{Converged}{Converged}
\SetKwFunction{Converged}{Converged}
\SetKwFunction{SolveCP}{SolveCP}
\KwData{$\B{x}^{(i)}_0$, $\B{x}^{(i)}_f$, $\Delta t$}
\KwResult{$\mathcal X^{(i)}_\mathrm{sol} = \left(\B{x}^{(i)}_0,\B{x}^{(i)}_1,\B{x}^{(i)}_2, \ldots, \B{x}^{(i)}_{T^{(i)}}\right)$, $\mathcal U^{(i)}_\mathrm{sol} = \left(\B{u}^{(i)}_0,\B{u}^{(i)}_1,\B{u}^{(i)}_2, \ldots, \B{x}^{(i)}_{T^{(i)}-1} \right)$}

\Comment{Stage 1: Find duration $t_f$ and initial trajectories that are close to the goal state}
$c^{(i)} \leftarrow \infty, \mathcal X^{(i)}_\mathrm{sol} \leftarrow (), \mathcal U^{(i)}_\mathrm{sol} \leftarrow ()\,\, \forall i \in \{1,\ldots,N\}$ \label{alg:planning:init}\;
\Repeat{\TerminationConditionOne{} \label{alg:planning:term_cond1}}{\label{alg:planning:ao_loop}
    \tikzmk{A}\ForEach{$i\in$ \RandomShuffle{$\{1,\ldots,N\}$}\markline{-14pt} \label{alg:planning:robot_loop} \label{alg:planning:rrt_start}}{
        $\mathcal T = (\{\B{x}^{(i)}_0\}, \emptyset)$\;
        \RepeatForever{
            $\B{x}_\mathrm{rand} \leftarrow$ \UniformSample{$\mathcal{X}^{\mathcal{I}(i)}$} \label{alg:planning:xrand}\;
            $\B{x}_\mathrm{near} \leftarrow$ \FindClosest{$\mathcal T$, $\B{x}_\mathrm{rand}$} \label{alg:planning:xnear}\;
            $\B{u}_\mathrm{rand} \leftarrow$ \UniformSample{$\mathcal{U}^{\mathcal{I}(i)}$} \label{alg:planning:urand}\;
            $\B{x}_\mathrm{new}, c \leftarrow$ \Propagate{$\B{x}_\mathrm{near}$, $\B{u}_\mathrm{rand}$, $\Delta t$, $\{\mathcal X^{(j)}_\mathrm{sol} | j \neq i\}$} \label{alg:planning:prop}\;
            \tikzmk{A}\If{\StateValid{$\B{x}_\mathrm{new}, \{\mathcal X^{(j)}_\mathrm{sol} | j \neq i\}$} and $c \leq c^{(i)}$ \label{alg:planning:state_valid}\markline{-29pt}}{
                \Add{$\mathcal T, \B{x}_\mathrm{near} \to \B{x}_\mathrm{new}$} \label{alg:planning:add}\;
                \If{$\|\B{x}_\mathrm{new} - \B{x}^{(i)}_f\| \leq \varepsilon$}{
                    $c^{(i)} \leftarrow c$ \label{alg:planning:cost_update}\;
                    $\mathcal X^{(i)}_\mathrm{sol}, \mathcal U^{(i)}_\mathrm{sol} \leftarrow$ \ExtractSolution{$\mathcal T, \B{x}_\mathrm{new}$}\;
                    break \label{alg:planning:extract}\;
                }
            }
        }
        \label{alg:planning:rrt_end}
    }
}
$\mathcal X^{(i)}_\mathrm{sol}, \mathcal U^{(i)}_\mathrm{sol} \leftarrow$ \PostProcess{$\mathcal X^{(i)}_\mathrm{sol}, \mathcal U^{(i)}_\mathrm{sol}$} \label{alg:planning:postprocess}\;
\Comment{Stage 2: Refine trajectories sequentially; Based on SCP}
\Repeat{\Converged{} \label{alg:planning:stage2_end}}{\label{alg:planning:stage2_start}
    \ForEach{$i\in$ \RandomShuffle{$\{1,\ldots,N\}$}}{
        \tikzmk{A}$\mathcal X^{(i)}_\mathrm{sol}, \mathcal U^{(i)}_\mathrm{sol} \leftarrow$ \SolveCP{Eq. \eqref{eq:scp}, $\{\mathcal X^{(i)}_\mathrm{sol} | \forall i\}, \{\mathcal U^{(i)}_\mathrm{sol} | \forall i\}$}\markline{-22pt}
    }  
}
\caption{Interaction-aware motion planning}
\label{alg:planning}
\end{algorithm}
\subsection{Stage 1: Sampling-Based Planning using Interaction Forces}
\label{sec:AO-RRT}
For Stage 1, any kinodynamic single-robot motion planner can be extended.
For the coupled multi-robot setting in the present paper, we modify AO-RRT~\cite{AO-RRT}, which is is a meta-algorithm that uses the rapidly-exploring random tree (RRT) algorithm as a subroutine.

\textbf{Sampling-Based Planner:}
Our adaption of RRT (\crefrange{alg:planning:rrt_start}{alg:planning:rrt_end} in \cref{alg:planning}) works as follows.
First, a random state $\B{x}_\mathrm{rand}$ is uniformly sampled from the state space (\cref{alg:planning:xrand}) and the closest state $\B{x}_\mathrm{near}$ that is already in the search tree $\mathcal T$ is found (\cref{alg:planning:xnear}).
This search can be done efficiently in logarithmic time by employing a specialized data structured such as a kd-tree~\cite{kdtree} and requires the user to define a distance function on the state space.
Second, an action is uniformly sampled from the action space (\cref{alg:planning:urand}) and the dynamics \eqref{eq:hetero-onerobot} are forward propagated for a fixed time period $\Delta t$ using $\B{x}_\mathrm{near}$ as the initial condition, e.g., by using the Runge-Kutta method (\cref{alg:planning:prop}).
Note that this forward propagation directly considers the learned dynamics $\favhat^{(i)}$. 
Third, the new state $\B{x}_\mathrm{new}$ is checked for validity with respect to i) the state space (which includes $\favhat^{(i)}$), ii) collisions with other robots, and iii) change and bound of the neighbor's interaction forces (\cref{alg:planning:state_valid}).
The first validity check ensures that the interaction force of the robot itself is bounded, while the third check is a trust region and upper bound for the neighbor's interaction forces.

If $\B{x}_\mathrm{new}$ is valid, it is added as a child node of $\B{x}_\mathrm{near}$ in the search tree $\mathcal T$ (\cref{alg:planning:add}).
Finally, if $\B{x}_\mathrm{new}$ is within an $\varepsilon$-distance to the goal $\B{x}^{(i)}_f$, the solution can be extracted by following the parent pointers of each tree node starting from $\B{x}_\mathrm{new}$ until the root node $\B{x}^{(i)}_{0}$ is reached (\cref{alg:planning:extract}).

We note that our RRT \changed{steering method} departs from ones in the literature which either sample $\Delta t$, use a best-control approximation of the steer method in RRT, or use a combination of both $\Delta t$-sampling and best-control approximation~\cite{AO-RRT}.
We are interested in a constant $\Delta t$ for our optimization formulation in Stage 2. In this case, a best-control approximation would lead to a probabilistic incomplete planner~\cite{kunzKinodynamicRRTsFixed2015}.
\changed{We adopt a technique of goal biasing where we pick the goal state rather than $\B{x}_\mathrm{rand}$ in fixed intervals, in order to improve the convergence speed.}

While RRT is \changed{probabilistically} complete, it also almost surely converges to a suboptimal solution~\cite{RRTstarAndPRMstar}.
AO-RRT remedies this shortcoming by planning in a state-cost space and using RRTs sequentially with a monotonically decreasing cost bound.
The cost bound $c^{(i)}$ is initially set to infinity (\cref{alg:planning:init}) and the tree can only grow with states that have a lower cost associated with them (\cref{alg:planning:state_valid}).
Once a solution is found, the cost bound is decreased accordingly (\cref{alg:planning:cost_update}) and the search is repeated using the new cost bound (\cref{alg:planning:ao_loop}).
This approach is asymptotically optimal, but in practice the algorithm is terminated based on some condition, e.g., a timeout or a fixed number of iterations without improvements (\cref{alg:planning:term_cond1}).

\begin{figure}
\includegraphics[width=\linewidth]{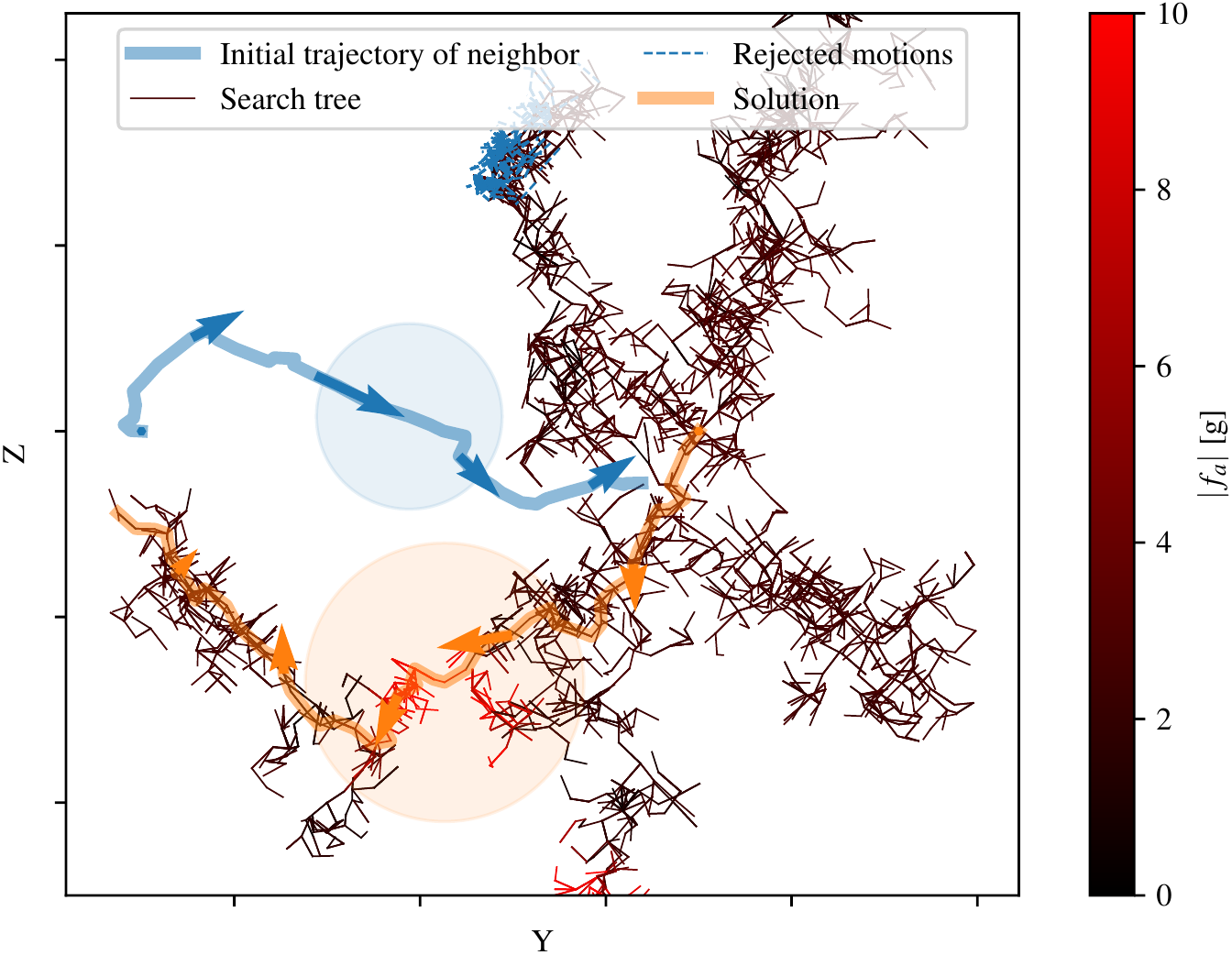}
\caption{Example for Stage 1 of our motion planning with learned dynamics. Here, we have an initial solution for a small (blue) robot and plan for a large (orange) robot. The created search tree of the large robot is color-coded by the magnitude of the interaction force on the orange robot. During the search, we reject states that would cause a significant change in the interaction force for the blue robot (edges in blue).}
\label{fig:plot6}
\end{figure}

\textbf{Modification of Sampling-Based Planner:} We extend AO-RRT to a sequential interaction-aware multi-robot planner by adding $\favhat^{(i)}$ and time to our state space and treating the other robots as dynamic obstacles.
As cost, we use a discrete approximation of the objective in \cref{eq:motion-planning}.
For each AO-RRT outer-loop iteration with a fixed cost bound, we compute trajectories sequentially using a random permutation of the robots (\cref{alg:planning:robot_loop}).
When we check the state for validity (\cref{alg:planning:state_valid}), we also enforce that the new state is not in collision with the trajectories of the other robots and that their interaction forces are bounded and within a trust region compared to their previous value, see \cref{fig:plot6} for visualization.
Here, the red edges show motions that cause large ($\approx \SI{10}{g}$) but admissible interaction forces on the orange robot, because the blue robot flies directly above it. The blue edges are candidate edges as computed in \cref{alg:planning:prop} and are not added to the search tree, because their motion would cause a violation of the interaction force trust region of the blue robot (condition in \cref{alg:planning:state_valid}). Once the search tree contains a path to the goal region, a solution is returned (orange path).

The output of the sequential planner (\cref{alg:planning:extract}) is a sequence of states $\mathcal X^{(i)}_\mathrm{sol}$ and actions $\mathcal U^{(i)}_\mathrm{sol}$, each to be applied for a duration of $\Delta t$. Note that the sequences might have different lengths for each robot. Implicitly, the sequences also defines $t_f$. Furthermore, the first element of each sequence is the robots' start state and the last element is within a $\varepsilon$-distance of the robots' goal state.
We postprocess this sequence of states to make it an appropriate input for the optimization, e.g., for uniform length (\cref{alg:planning:postprocess}). In practice, we found that repeating the last state and adding null actions, or (virtual) tracking of the computed trajectories using a controller are efficient and effective postprocessing techniques.

Other sampling-based methods can be used as foundation of the first stage as well, with similar changes in sequential planning, state-augmentation to include the interaction forces, state-validity checking, and postprocessing.

\subsection{Stage 2: Optimization-Based Motion Planning}

We employ sequential convex programming (SCP) for optimization.
SCP is a local optimization method for nonconvex problems that leverages convex optimization. The key concept is to convexify the nonconvex portions of the optimization problem by linearizing around a prior solution. The resulting convex problem instance is solved and a new solution obtained. The procedure can be repeated until convergence criteria are met.
Because of the local nature of this procedure, a good initial guess is crucial for high-dimensional and highly nonlinear system dynamics. In our case, we use the searched trajectories from Stage 1 in \Cref{sec:AO-RRT} as the initial guess.

We first adopt a simple zero-order hold temporal discretization of the dynamics \eqref{eq:planning-dynamics} using Euler integration:
\begin{equation}
    \B{x}^{(i)}_{k+1} = \B{x}^{(i)}_k + \dot{\B{x}}^{(i)}_k \Delta t.
    \label{eq:planning-discreteDynamics}
\end{equation}

Second, we linearize $\dot{\B{x}}^{(i)}_k$ around prior states $\bar{\B{x}}^{(i)}_k$ and actions $\bar{\B{u}}^{(i)}_k$:
\begin{equation}
    \dot{\B{x}}^{(i)}_k \approx \B{A}_k (\B{x}^{(i)}_k - \bar{\B{x}}^{(i)}_k) + \B{B}_k (\B{u}^{(i)}_k - \bar{\B{u}}^{(i)}_k) + \B{f}^{(i)}(\bar{\B{x}}^{(i)}_k, \bar{\B{u}}^{(i)}_k),
    \label{eq:planning-xdot-approx}
\end{equation}
where $\B{A}_k$ and $\B{B}_k$ are the partial derivative matrices of $\B{f}^{(i)}$ with respect to $\B{x}^{(i)}_k$ and $\B{u}^{(i)}_k$ evaluated at $\bar{\B{x}}^{(i)}_k, \bar{\B{u}}^{(i)}_k$.
Because we encode $\favhat^{(i)}$ using fully-differentiable DNNs, the partial derivatives can be efficiently computed analytically, e.g., by using \texttt{autograd} in PyTorch~\cite{pyTorch}.

Third, we linearize $\favhat^{(j)}$ around our prior states $\bar{\B{x}}^{(i)}_k$ for all neighboring robots $j\in\mathrm{neighbor}(i)$:
\begin{equation}
    \hat{\B{f}}_{a}^{(j)} \approx \B{C}^{(j)}_k (\B{x}^{(i)}_k - \bar{\B{x}}^{(i)}_k) + \favhat^{(j)}(\set^{(i)}_{\mathcal{I}}(\bar{\B{x}}^{(i)}_k)),
    \label{eq:planning-Fa-approx}
\end{equation}
where $\B{C}^{(j)}_k$ is the derivative matrix of $\favhat^{(j)}$ (the learned interaction function of robot $j$, represented by DNNs) with respect to $\B{x}^{(i)}_k$ evaluated at $\bar{\B{x}}^{(i)}_k$; and $\set^{(i)}_{\mathcal{I}}(\bar{\B{x}}^{(i)}_k)$ is the ordered sequence of relative states as defined in \eqref{eq:all-rel-states} but using the fixed prior state $\bar{\B{x}}^{(i)}_k$ rather than decision variable $\B{x}^{(i)}_k$ in \eqref{eq:rel-states}.

\changed{We now formulate a convex program, one per robot:}
\begin{align}
&\min_{\mathcal X^{(i)}_\mathrm{sol},\mathcal U^{(i)}_\mathrm{sol}} \sum_{t=0}^{T}\| \B{u}^{(i)}_k \|^2 + \lambda_1 \| \B{x}^{(i)}_{T} - \B{x}_f^{(i)} \|_{\infty} + \lambda_2 \delta \label{eq:scp}
\\
&\quad\text{subject to:} \nonumber \\
&\begin{cases}
\text{robot dynamics \eqref{eq:planning-discreteDynamics} and \eqref{eq:planning-xdot-approx}} &i\in[1,N] \\
\B{u}^{(i)}_k \in \mathcal{U}^{\mathcal{I}(i)} &i\in[1,N] \\
\B{x}^{(i)}_k \in \mathcal{X}^{\mathcal{I}(i)}_\delta &i\in[1,N],\delta \geq 0 \\
\langle \bar{\B{p}}^{(ij)}_k, \B{p}^{(i)}_k - \bar{\B{p}}^{(i)}_k \rangle \geq r^{(\mathcal{I}(i) \mathcal{I}(j))} \| \bar{\B{p}}^{(ij)}_k\|_2 &i<j,j\in[2,N]\\
\B{x}^{(i)}_0 = \B{x}_s^{(i)} &i\in[1,N]\\
| \B{C}^{(j)}_k (\B{x}^{(i)}_k - \bar{\B{x}}^{(i)}_k) | \leq b_{fa} &i<j,j\in[2,N]\\
| \B{x}^{(i)}_k - \bar{\B{x}}^{(i)}_k | \leq \B{b}_{x}; \;
| \B{u}^{(i)}_k - \bar{\B{u}}^{(i)}_k | \leq \B{b}_{u} \; &i\in[1,N]
\end{cases} \nonumber
\end{align}
where $\mathcal{X}^{\mathcal{I}(i)}_\delta$ is the state space increased by $\delta$ in each direction, the linearized robot dynamics are similar to~\cite{foustOptimalGuidanceControl2020,nakka2020chance}, and the convexified inter-robot collision constraint is from~\cite{morgan2016swarm}.
We use soft constraints for reaching the goal (with weight $\lambda_1$) and the state space (with weight $\lambda_2$), and trust regions around $\bar{\B{x}}^{(i)}_k$, $\bar{\B{u}}^{(i)}_k$, and the neighbors' interaction forces for numerical stability.
Interaction forces are constrained in \eqref{eq:scp} because $\hat{\B{f}}_{a}^{(i)}$ is part of the state space $\mathcal{X}^{\mathcal{I}(i)}$.

We solve these convex programs sequentially and they converge to a locally optimal solution~\cite{morgan2016swarm}.
For the first iteration, we linearize around the trajectory computed during Stage 1 of our motion planner while subsequent iterations linearize around the solution of the previous iteration (\crefrange{alg:planning:stage2_start}{alg:planning:stage2_end} in \cref{alg:planning}).
It is possible to implement \cref{alg:planning} in a distributed fashion similar to prior work~\cite{morgan2016swarm}.

\section{Interaction-Aware Tracking Controller}
Given arbitrary desired trajectories, including ones that have not been computed using the method presented in \cref{sec:planning}, we augment existing nonlinear position and attitude controllers for multirotors~\cite{bandyopadhyay2016nonlinear,shi2018nonlinear} that account for interaction forces and torques and compensate motor delays.
\subsection{Tracking Control Law with Delay Compensation}
We use a typical hierarchical structure as shown in~\cref{fig:control-diagram} for controlling multirotor robots. Given the desired 3\changed{D} position trajectory $\B{p}_{d}^{(i)}(t)$ for robot $i$, we define a reference velocity
\begin{equation}
    \B{v}_{r}^{(i)} = \dot{\B{p}}_{d}^{(i)} - \bm{\Lambda}_p^{(i)} \tilde{\B{p}}^{(i)},\label{eq:v_ref}
\end{equation}
with position error $\tilde{\B{p}}^{(i)} = \B{p}^{(i)} - \B{p}_{d}^{(i)}$ and gain matrix $\bm{\Lambda}_p^{(i)} \succ 0$. The position controller is defined by the desired thrust vector
\begin{equation}
\begin{split}
        \B{f}_d^{(i)} &= -m^{(i)}\B{g} + m^{(i)}\dot{\B{v}}_r^{(i)} - \favhat^{(i)} \\
        & \qquad - \left(\B{K}_{v}^{(i)} + m \bm{\Gamma}_{v}^{(i)}\right)\tilde{\B{v}}^{(i)}  -\B{K}_{v}^{(i)}\bm{\Gamma}_{v}^{(i)}\int{\tilde{\B{v}}^{(i)}}, \label{eq:pos_ctrl}
\end{split}
\end{equation}
where $\tilde{\B{v}}^{(i)} = \B{v}^{(i)} -  \B{v}_{r}^{(i)}$ is the velocity error, and $\B{K}_{v}^{(i)}$, $\bm{\Gamma}_{v}^{(i)}$ are positive definite gain matrices. From~\cref{eq:pos_dynamics}, by setting $\B{R}^{(i)}\B{f}_u^{(i)} = \B{f}_d^{(i)}$, we compute the total desired thrust $T_d^{(i)} = \B{f}_d^{(i)} \cdot \hat{k}$ and the desired attitude $\B{R}_d^{(i)}$~\cite{shi2019neural-lander}. We convert the error rotation matrix $\tilde{\B{R}}^{(i)} = \B{R}_d^{(i)\top} \B{R}^{(i)}$ to a constrained quaternion error $\tilde{\B{q}}^{(i)} = [\tilde{q}_0^{(i)}, \tilde{\B{q}}_v^{(i)}]$, and define the reference angular rate as
\begin{equation}
    \bm{\omega}_r^{(i)} = \tilde{\B{R}}^{(i)\top}\bm{\omega}_d^{(i)} -\bm{\Lambda}_q^{(i)} \tilde{\B{q}}_v^{(i)}.\label{eq:omega_ref}
\end{equation}
We use the following nonlinear attitude controller from~\cite{bandyopadhyay2016nonlinear} with interaction torque compensation:
\begin{multline}
    \bm{\tau}_d^{(i)} = \B{J}^{(i)}\dot{\bm{\omega}}_r^{(i)} - \B{J}^{(i)}\bm{\omega}^{(i)} \times \bm{\omega}_r^{(i)} - \tauavhat^{(i)} \\
    - \B{K}_{\omega}^{(i)} \tilde{\bm{\omega}}^{(i)} - \bm{\Gamma}_{\omega}^{(i)} \int \tilde{\bm{\omega}}^{(i)}. \label{eq:att_ctrl}
\end{multline}
$\B{K}_{\omega}^{(i)}$ and $\bm{\Gamma}_{\omega}^{(i)}$ are positive definite gain matrices on angular rate error $\tilde{\bm{\omega}}^{(i)} = \bm{\omega}^{(i)} - \bm{\omega}_r^{(i)}$ and its integral, respectively.

From~\cref{eq:pos_ctrl,eq:att_ctrl}, the desired output wrench for the $i$-th robot $\bm{\eta}_d^{(i)} = \left[ T_d^{(i)}; \bm{\tau}_d^{(i)}\right]$ must be realized through a delayed motor signal $\B{u}_c^{(i)}$ from~\cref{eq:delay_model}. Here, we implement a simple yet effective method to compensate for motor delay~\cite{shi2020numerical}:
\begin{equation}
    \B{u}_c^{(i)} = \B{B}_0^{(i)+} \left(\bm{\eta}_d^{(i)} + \frac{\dot{\bm{\eta}}_d^{(i)}}{\lambda^{(i)}} \right),
    \label{eq:delay_ctrl}
\end{equation}
where the actuation matrix $\B{B}_0^{(i)}$ and delay constant $\lambda^{(i)}$ are determined a priori. We consider the multirotor to be fully or over-actuated, thus $\left(\cdot\right)^+$ denotes either the inverse or right pseudo-inverse. $\dot{\bm{\eta}}_d^{(i)}$ can be obtained through numerical differentiation~\cite{shi2020numerical}.

\begin{figure}
\includegraphics[width=\linewidth]{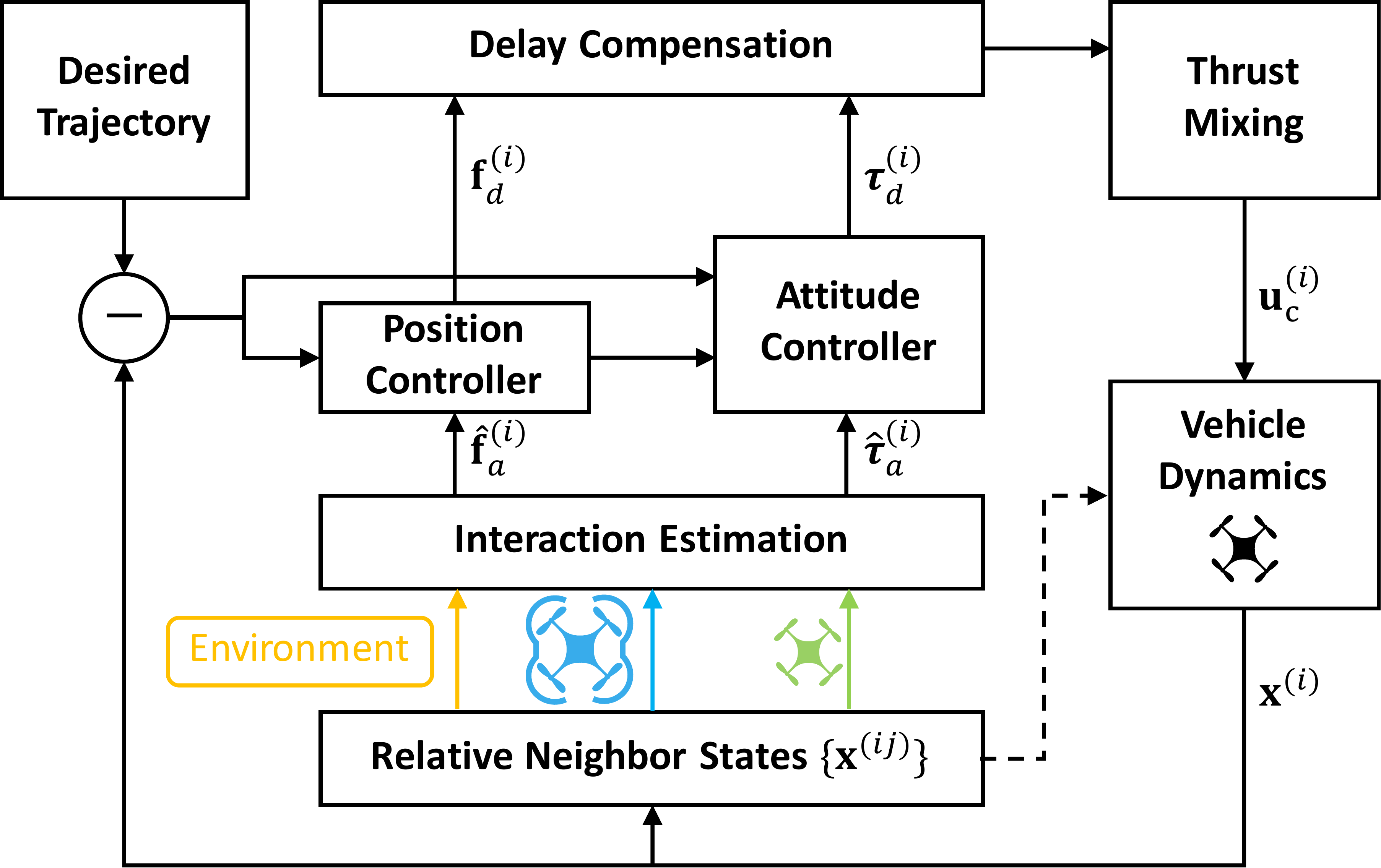}
\caption{Hierarchy of control and planning blocks with information flow for commands and sensor feedback. We use different colors to represent heterogeneous neighbors. Note that the neighbors will influence the vehicle dynamics (dashed arrow).}
\label{fig:control-diagram}
\end{figure}

\subsection{Analysis of Stability and Robustness}
\label{sec:stability}
The robust position and attitude controllers \cref{eq:pos_ctrl,eq:att_ctrl} can handle bounded model disturbance~\cite{bandyopadhyay2016nonlinear,shi2019neural-lander}. Here, we make the same assumption as in~\cite{shi2019neural-lander,shi2020neural}, that the learning errors and their rates of change are upper bounded.

\begin{assumption}[Bounded approximation error of DNNs]
\label{assump:bounded_error}
We denote the approximation errors between the learned model in \cref{eq:learningmodel-hetero} and the true unmodeled dynamics for interaction force and torque as $\bm{\epsilon}_f = \fav^{(i)} - \favhat^{(i)}$ and $\bm{\epsilon}_\tau = \tauav^{(i)} - \tauavhat^{(i)}$, respectively. For each robot, we assume $\bm{\epsilon}_f$ and $\bm{\epsilon}_\tau$ are uniformly upper bounded. Formally, $\sup_{\B{x}^{(i)} \in \mathcal{X}^{\mathcal{I}(i)}} \lVert \bm{\epsilon}_f \rVert = \bar{\epsilon}_{f}^{(i)} $ and $\sup_{\B{x}^{(i)} \in \mathcal{X}^{\mathcal{I}(i)}} \lVert \bm{\epsilon}_\tau \rVert = \bar{\epsilon}_{\tau}^{(i)}$. Furthermore, we assume the time derivative of errors are upper bounded as well, i.e $\sup_{\B{x}^{(i)} \in \mathcal{X}^{\mathcal{I}(i)}} \lVert \dot{\bm{\epsilon}}_f \rVert = \bar{d}_{f}^{(i)} $ and $\sup_{\B{x}^{(i)} \in \mathcal{X}^{\mathcal{I}(i)}} \lVert \dot{\bm{\epsilon}}_\tau \rVert = \bar{d}_{\tau}^{(i)}$.
\end{assumption}

\changed{Note that when the number of agents are fixed, the $\bar{\epsilon}_f^{(i)}$, $\bar{\epsilon}_\tau^{(i)}$, $\bar{d}_{f}^{(i)}$, and $\bar{d}_{\tau}^{(i)}$ can all be derived from the Lipschitz constant  in \cref{eq:learningmodel-hetero} for spectrally-normalized DNNs~\cite{bartlett2017spectrally,liu2020robust}, under standard training data distribution assumptions. It is common to assume such bounded approximation errors in learning-based control, e.g., \cite{shi2019neural-lander,taylor2019episodic,cheng2019control,mckinnon2019learn}.} Under Assumption \ref{assump:bounded_error}, we can show the stability of the position and attitude controllers using results from~\cite{shi2020neural,bandyopadhyay2016nonlinear}:

\begin{theorem}
For the position controller defined in~\cref{eq:v_ref,eq:pos_ctrl} under Assumption~\ref{assump:bounded_error}, the position tracking error $\lVert\tilde{\B{p}}^{(i)}\rVert$ converges exponentially to an error ball: 
\begin{equation}
    \lim_{t \to \infty}{\lVert\tilde{\B{p}}^{(i)}\rVert} = \frac{\bar{d}_{f}^{(i)}}{\lambda_{\min}(\bm{\Lambda}_p^{(i)})\lambda_{\min}(\bm{\Gamma}_v^{(i)})\lambda_{\min}(\B{K}_{v}^{(i)})}
    \label{eq:p1}
\end{equation}
\label{thm:p1}
\end{theorem}

\begin{proof}
We select sliding variables: $\B{s}_1 = \dot{\tilde{\B{p}}}^{(i)} + \bm{\Lambda}_p^{(i)} \tilde{\B{p}}^{(i)}$ and $\B{s}_2 = m \dot{\B{s}}_1 + \B{K}_v \B{s}_1$. Then, \cref{eq:pos_ctrl} can be written as
\begin{equation*}
     \B{f}_d^{(i)} = -m^{(i)}\B{g} + m^{(i)}\dot{\B{v}}_r^{(i)} - \favhat^{(i)} - \B{K}_{v}^{(i)}\B{s}_1 -\bm{\Gamma}_v^{(i)}\int{\B{s}_2}.
\end{equation*}
Applying to~\cref{eq:pos_dynamics}, we can get closed-loop dynamics $\dot{\B{s}}_2 + \bm{\Gamma}_v^{(i)} \B{s}_2 = \dot{\bm{\epsilon}}_f$. Thus combining hierarchical linear systems of $\dot{\B{s}}_1$ and $\dot{\B{s}}_2$, we can easily arrive at~\cref{eq:p1}.
\end{proof}

\begin{theorem}
For the attitude controller defined in~\cref{eq:omega_ref,eq:att_ctrl} under Assumption~\ref{assump:bounded_error}, the attitude tracking error $\lVert\tilde{\B{q}}_v^{(i)}\rVert$ converges exponentially to an error ball determined by $\bar{d}_{\tau}^{(i)}$.
\label{thm:p2}
\end{theorem}

\begin{proof}
Applying~\cref{eq:att_ctrl} to~\cref{eq:att_dynamics}, we get closed-loop dynamics
\begin{equation*}
    \B{J}^{(i)}\dot{\tilde{\bm{\omega}}}^{(i)} - \B{J}^{(i)}\bm{\omega}^{(i)} \times \tilde{\bm{\omega}}^{(i)}
    - \B{K}_{\omega}^{(i)}\tilde{\bm{\omega}}^{(i)} - \bm{\Gamma}_{\omega}^{(i)} \int \tilde{\bm{\omega}}^{(i)} = \bm{\epsilon}_\tau.
\end{equation*}
Following the proof structure for Theorem 2 in~\cite{bandyopadhyay2016nonlinear}, we can derive an ultimate bound for $\lVert\tilde{\B{q}}_v^{(i)}\rVert$ determined by $\bar{d}_{\tau}^{(i)}$.
\end{proof}

With motor delay, we can state the following result for stabilizing the output wrench error $\tilde{\bm{\eta}}^{(i)} = \bm{\eta}^{(i)} - \bm{\eta}^{(i)}_d$ with~\cref{eq:delay_ctrl}, assuming the motor delay constant is obtained from testing.

\begin{theorem}
For robot $i$, the controllers~\cref{eq:pos_ctrl,eq:att_ctrl,eq:delay_ctrl} will exponentially stabilize the augmented states of position, attitude and output wrench error: $[\tilde{\B{p}}^{(i)};\tilde{\B{v}}^{(i)};\tilde{\B{q}}_v^{(i)};\tilde{\bm{\omega}}^{(i)};\tilde{\bm{\eta}}^{(i)}]$.
\end{theorem}
\begin{proof}
\Cref{eq:pos_ctrl,eq:att_ctrl} stabilizes $[\tilde{\B{p}}^{(i)};\tilde{\B{v}}^{(i)};\tilde{\B{q}}_v^{(i)};\tilde{\bm{\omega}}^{(i)}]$ exponentially from~\cref{thm:p1,thm:p2}. Thus by Theorem 3.1 from~\cite{shi2020numerical}, it follows that the augmented states is also exponentially stabilized by~\cref{eq:delay_ctrl}.
\end{proof}

With small modelling errors on $\lambda^{(i)}$, the controller~\cref{eq:delay_ctrl} can robustly cancel out some effects from delays and improve tracking performance in practice. Furthermore, it can handle not only first-order motor delay~\cref{eq:delay_model}, but also signal transport delays~\cite{shi2020numerical}. In case of the small quadrotors used in our experiments, such delays are on the same order of magnitude as the motor delay, thus making~\cref{eq:delay_ctrl} essential for improving the control performance.

\section{Experiments}
We use quadrotors based on Bitcraze Crazyflie 2.0/2.1 (CF). Our small quadrotors are Crazyflie 2.X, which are small (\SI{9}{cm} rotor-to-rotor) and lightweight (\SI{34}{g}) products that are commercially available.
Our large quadrotors use the Crazyflie 2.1 as control board on a larger frame with brushed motors (model: Parrot Mini Drone), see \cref{tab:cfprops} for a summary of physical parameters.
We use the Crazyswarm~\cite{crazyswarm} package to control multiple Crazyflies simultaneously.
Each quadrotor is equipped with a single reflective marker for position tracking at \SI{100}{Hz} using a motion capture system.
The nonlinear controller, extended Kalman filter, and neural network evaluation are running on-board the STM32 microcontroller.

For the controller, we implement the delay compensation \eqref{eq:delay_ctrl} in the following way: i) we numerically estimate $\dot{T}_d^{(i)}$ as part of the position controller \eqref{eq:pos_ctrl}, and ii) we approximate $\dot{\bm{\tau}}_d^{(i)}$ by adding the additional term $-\B{K}_{\dot{\omega}}^{(i)}\dot{\tilde{\bm{\omega}}}^{(i)}$ to the attitude controller \eqref{eq:att_ctrl}, where $\dot{\tilde{\bm{\omega}}}^{(i)}$ is numerically estimated and $\B{K}_{\dot{\omega}}^{(i)}$ is a positive definite gain matrix. We found that the other terms of $\dot{\bm{\tau}}_d^{(i)}$, i.e. the time-derivative of \eqref{eq:att_ctrl}, are negligible for our use-case.
\changed{The baseline controller is identical (including the chosen gains) to our proposed controller except that the interaction force for the baseline is set to zero.
The baseline controller is much more robust and efficient than the well-tuned nonlinear controller in the Crazyswarm package, which cannot safely execute the close-proximity flight shown in \cref{fig:fig1}(c) and requires at least \SI{60}{cm} safety distance~\cite{DBLP:journals/trob/HonigPKSA18}.}

For data collection, we use the $\mu$SD card extension board and store binary encoded data roughly every \SI{10}{ms}.
Each dataset is timestamped using the on-board microsecond timer and the clocks are synchronized before takeoff using broadcast radio packets.
The drift of the clocks of different Crazyflies can be ignored for our short flight times (less than \SI{2}{min}).

\begin{table}
\caption{System Identification of the used quadrotors.}
\label{tab:cfprops}
\centering
\begin{tabular}{C{1.2cm}||C{3cm}|C{3cm}}
& Small & Large \\
& \includegraphics[width=1.0\linewidth]{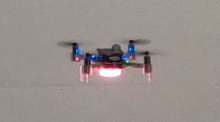} & \includegraphics[width=1.0\linewidth]{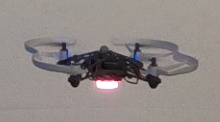}
\\
\hline
\hline
Weight & \SI{34}{g} & \SI{67}{g}\\
Max Thrust & \SI{65}{g} & \SI{145}{g}\\
Diameter & \SI{12}{cm} & \SI{19}{cm} \\
$\lambda$ & 16 & 16\\
$\psi_1(\hat{p},\hat{v})$ & $11.09 -39.08\hat{p}-9.53\hat{v}+20.57\hat{p}^2+38.43\hat{p}\hat{v}$ & $44.1.0 -122.51\hat{p}-36.18\hat{v}+53.11\hat{p}^2+107.68\hat{p}\hat{v}$ \\
$\psi_2(\hat{f}, \hat{v})$ & $0.5 +0.12\hat{f}-0.41\hat{v}-0.002\hat{f}^2-0.043\hat{f}\hat{v}$ & $0.56 +0.06\hat{f}-0.6\hat{v}-0.0007\hat{f}^2-0.015\hat{f}\hat{v}$ \\
$\psi_3(\hat{p},\hat{v})$ & $-9.86 +3.02\hat{p}-26.72\hat{v}$ & $-29.91 +8.1\hat{p}+65.2\hat{v}$
\end{tabular}
\end{table}

\subsection{Calibration and System Identification of Different Robots}
\label{sec:systemid}

Prior to learning the residual terms $\fav^{(i)}$ and $\tauav^{(i)}$, we first calibrate the nominal dynamics model $\fnom^{(i)}(\B{x},\B{u})$.
We found that existing motor thrust models \cite{BitcrazeThrust, systemid} are not very accurate, because they only consider a single motor and ignore the effect of the battery state of charge.
We calibrate each Crazyflie by mounting the whole quadrotor upside-down on a load cell (model TAL221) which is directly connected to the Crazyflie via a custom extension board using a 24-bit ADC (model HX711).
The upside-down mounting avoids contamination of our measurements with downwash-related forces.
We use a \SI{100}{g} capacity load cell for the small quadrotor and a \SI{500}{g} capacity load cell for the large quadrotor.
We randomly generate desired PWM motor signals (identical for all 4 motors) and collect the current battery voltage, PWM signals, and measured force.
We use this data to find three polynomial functions: $\psi_1$, $\psi_2$, and $\psi_3$.
The first $\hat{f} = \psi_1(\hat{p},\hat{v})$ computes the force of a single rotor given the normalized PWM signal $\hat{p}$ and the normalized battery voltage $\hat{v}$.
This function is only required for the data collection preparation in order to compute $\fav^{(i)}$.
The second $\hat{p} = \psi_2(\hat{f}, \hat{v})$ computes the required PWM signal $\hat{p}$ given the desired force $\hat{f}$ and current battery voltage $\hat{v}$.
Finally, $\hat{f}_\mathrm{max} = \psi_3(\hat{p},\hat{v})$ computes the maximum achievable force $\hat{f}_\mathrm{max}$, given the current PWM signal $\hat{p}$ and battery voltage $\hat{v}$.
The last two functions are important at runtime for outputting the correct force as well as for thrust mixing when motors are saturated~\cite{DBLP:journals/ral/FaesslerFS17}.

We use the same measurement setup with the load cell to establish the delay model of $T_d^{(i)}$ with a square wave PWM signal.
While the delay model is slightly asymmetric in practice, we found that our symmetric model \cref{eq:delay_model} is a good approximation.
All results are summarized in \cref{tab:cfprops}.
We use the remaining parameters ($\B{J}$, thrust-to-torque ratio) from the existing literature~\cite{systemid}.

\subsection{Data Collection}
\label{exp:dataCollection}

\begin{table*}
\caption{12 scenarios for data collection.}
\label{tab:data_collection}
\centering
\begin{tabular}{c|C{3.75cm}C{3.75cm}C{3.75cm}C{3.75cm}}
Scenario & S & S$\,\rightarrow\,$S & L$\,\rightarrow$\,S & \{S,\,S\}$\,\rightarrow\,$S \\ \hline
Model & $\bm{\rho}_{\sm}(\bm{\phi}_{\env})$ & $\bm{\rho}_{\sm}(\bm{\phi}_{\env}+\bm{\phi}_{\sm})$ & $\bm{\rho}_{\sm}(\bm{\phi}_{\env}+\bm{\phi}_{\la})$ &
$\bm{\rho}_{\sm}(\bm{\phi}_{\env}+\bm{\phi}_{\sm}+\bm{\phi}_{\sm})$
\end{tabular}\\[2mm]
\begin{tabular}{c|C{3.75cm}C{3.75cm}C{3.75cm}C{3.75cm}}
Scenario & \{S,\,L\}$\,\rightarrow\,$S & \{L,\,L\}$\,\rightarrow\,$S & L & S$\,\rightarrow\,$L \\ \hline
Model & $\bm{\rho}_{\sm}(\bm{\phi}_{\env}+\bm{\phi}_{\sm}+\bm{\phi}_\la)$ &
$\bm{\rho}_{\sm}(\bm{\phi}_\env+\bm{\phi}_\la+\bm{\phi}_\la)$ &
$\bm{\rho}_{\la}(\bm{\phi}_\env)$ &
$\bm{\rho}_{\la}(\bm{\phi}_\env+\bm{\phi}_\sm)$
\end{tabular}\\[2mm]
\begin{tabular}{c|C{3.75cm}C{3.75cm}C{3.75cm}C{3.75cm}}
Scenario & L$\,\rightarrow\,$L & \{S,\,S\}$\,\rightarrow\,$L & \{S,\,L\}$\,\rightarrow\,$L & \{L,\,L\}$\,\rightarrow\,$S \\ \hline
Model & $\bm{\rho}_{\la}(\bm{\phi}_\env+\bm{\phi}_\la)$ &
$\bm{\rho}_{\la}(\bm{\phi}_\env+\bm{\phi}_\sm+\bm{\phi}_\sm)$ &
$\bm{\rho}_{\la}(\bm{\phi}_\env+\bm{\phi}_\sm+\bm{\phi}_\la)$ &
$\bm{\rho}_{\la}(\bm{\phi}_\env+\bm{\phi}_\la+\bm{\phi}_\la)$
\end{tabular}
\end{table*}

Recall that in \cref{eq:learningmodel-hetero}, we need to learn $2K$ neural networks \changed{for} $K$ types of robots. In our experiments, we consider two types of quadrotors (small and large) and also the environment (mainly ground effect and air drag), as shown in Example~\ref{example:hetero-system}. Therefore, we have 5 neural networks to be learned:
\begin{equation}
\label{eq:five-nns}
\bm{\rho}_{\sm},\,\bm{\rho}_{\la},\,\bm{\phi}_{\sm},\,\bm{\phi}_{\la},\,\bm{\phi}_{\env},
\end{equation}
where we do not have $\bm{\rho}_{\env}$ because the aerodynamical force acting on the environment is not interesting for our purpose. 
To learn these 5 neural networks, we fly the heterogeneous swarm in 12 different scenarios (see \cref{tab:data_collection}) to collect labeled $\fav^{(i)}$ and $\tauav^{(i)}$ data for each robot. For instance, Example~\ref{example:hetero-system} (as depicted in \cref{fig:fig1}(a)) corresponds to the ``\{S,\,S\}$\,\rightarrow\,$L'' scenario in \cref{tab:data_collection}, where the large robot has two small robots and the environment as its neighbors. 

We utilize two types of data collection tasks: random walk and swapping.
For random walk, we implement a simple reactive collision avoidance approach based on artificial potentials on-board each Crazyflie~\cite{artificialPotentials}.
The host computer randomly selects new goal points within a cube for each vehicle at a fixed frequency. These goal points are used as an attractive force, while neighboring drones contribute a repulsive force.
For swapping, the drones are placed in different horizontal planes on a cylinder and tasked to move to the opposite side. All the drones are vertically aligned for one time instance, causing a large interaction force. 
The random walk data helps us to explore the whole space quickly, while the swapping data ensures that we have data for a specific task of interest. Note that for both random walk and swapping, the drones also move close to the ground, to collect sufficient data for learning the ground effect. \changed{The collected data covers drone flying speeds from 0 to \SI{2}{m/s}, where $7\%$ are with relatively high speeds ($\geq$\SI{0.5}{m/s}) to learn the aerodynamic drag.} For both task types, we varied the scenarios listed in Table~\ref{tab:data_collection}.

To learn the 5 DNNs in~\cref{eq:five-nns}, for each robot $i$ in each scenario, we collect the timestamped states $\B{x}^{(i)}=[\B{p}^{(i)};\B{v}^{(i)};\B{R}^{(i)};\bm{\omega}^{(i)}]$.
We then compute $\B{y}^{(i)}$ as the observed value of $\fav^{(i)}$ and $\tauav^{(i)}$. We compute $\fav^{(i)}$ and $\tauav^{(i)}$  using \cref{eq:hetero-onerobot}, where the nominal dynamics $\fnom^{(i)}$ is calculated based on our system identification in \cref{sec:systemid}. \changed{With $\fnom^{(i)}$, $\B{y}^{(i)}$ is computed by $\dot{\B{x}}^{(i)}-\fnom^{(i)}$, where $\dot{\B{x}}^{(i)}$ is estimated by the five-point numerical differentiation method.} Note that the control delay \changed{$\lambda^{(i)}$} is also considered when we compute $\fav^{(i)}$ and $\tauav^{(i)}$.
Our training data consists of sequences of $\left(\{\set_{\mathrm{type}_1}^{(i)},\cdots,\set_{\mathrm{type}_K}^{(i)}\},\B{y}^{(i)}\right)$ pairs, where $\set^{(i)}_{\mathrm{type}_k} = \{\B{x}^{(ij)}|j\in\mathrm{neighbor}(i)\text{ and }\mathcal{I}(j)=\mathrm{type}_k\}$ is the set of the relative states of the type-$k$ neighbors of $i$. We have the following loss function for robot $i$ in each scenario (see \cref{tab:data_collection} for the detailed model structure in each scenario):  
\begin{equation}
\Bigg\| \bm{\rho}_{\mathcal{I}(i)}\Bigl(\sum_{k=1}^K\sum_{\B{x}^{(ij)}\in \set_{\mathrm{type}_k}^{(i)}}\bm{\phi}_{\mathcal{I}(j)}(\B{x}^{(ij)})\Bigr) - \B{y}^{(i)} \Bigg\|_2^2,
\end{equation}
and we stack all the robots' data in all scenarios and train on them together. \changed{There are 1.4 million pairs in the full dataset.}

\changed{In practice, we found the unmodeled torque $\|\tauav^{(i)}\|$ is very small (smaller by two orders of magnitude than the feedback term $\B{K}_{\omega}^{(i)}\tilde{\bm{\omega}}^{(i)}$ in the attitude controller \cref{eq:att_ctrl}), so we only learn $\fav^{(i)}$.} 
We compute the relative states from our collected data as $\B{x}^{(ij)}=[\B{p}^{(j)}-\B{p}^{(i)};\B{v}^{(j)}-\B{v}^{(i)}]\in\mathbb{R}^6$ (i.e., relative position and relative velocity both in the world frame), since the attitude information $\B{R}$ and $\bm{\omega}$ are not dominant for $\fav^{(i)}$. If the type of neighbor $j$ is ``environment'', we set $\B{p}^{(j)}=\B{0}$ and $\B{v}^{(j)}=\B{0}$. \changed{In this work, we only learn the $z$-component of $\fav^{(i)}$ since we found the other two components, $x$ and $y$, are much smaller and less varied, and do not significantly alter the nominal dynamics. In data collection, the rooted means and standard deviations of the squared values of $f_{a,x},f_{a,y}$, and $f_{a,z}$ are $1.6\pm2.5,1.2\pm2.2,5.0\pm8.9$ grams, respectively (for reference, the weights of the small and large drones are \SI{34}{g} and \SI{67}{g}).} Therefore, the output of our learning model in \cref{eq:learningmodel-hetero} is a scalar to approximate the $z$\changed{-}component of the unmodeled force function $\fav^{(i)}$.

\subsection{Learning Results and Ablation Analysis}
\label{sec:ablation}

\begin{figure*}
\centering
\includegraphics[width=\linewidth]{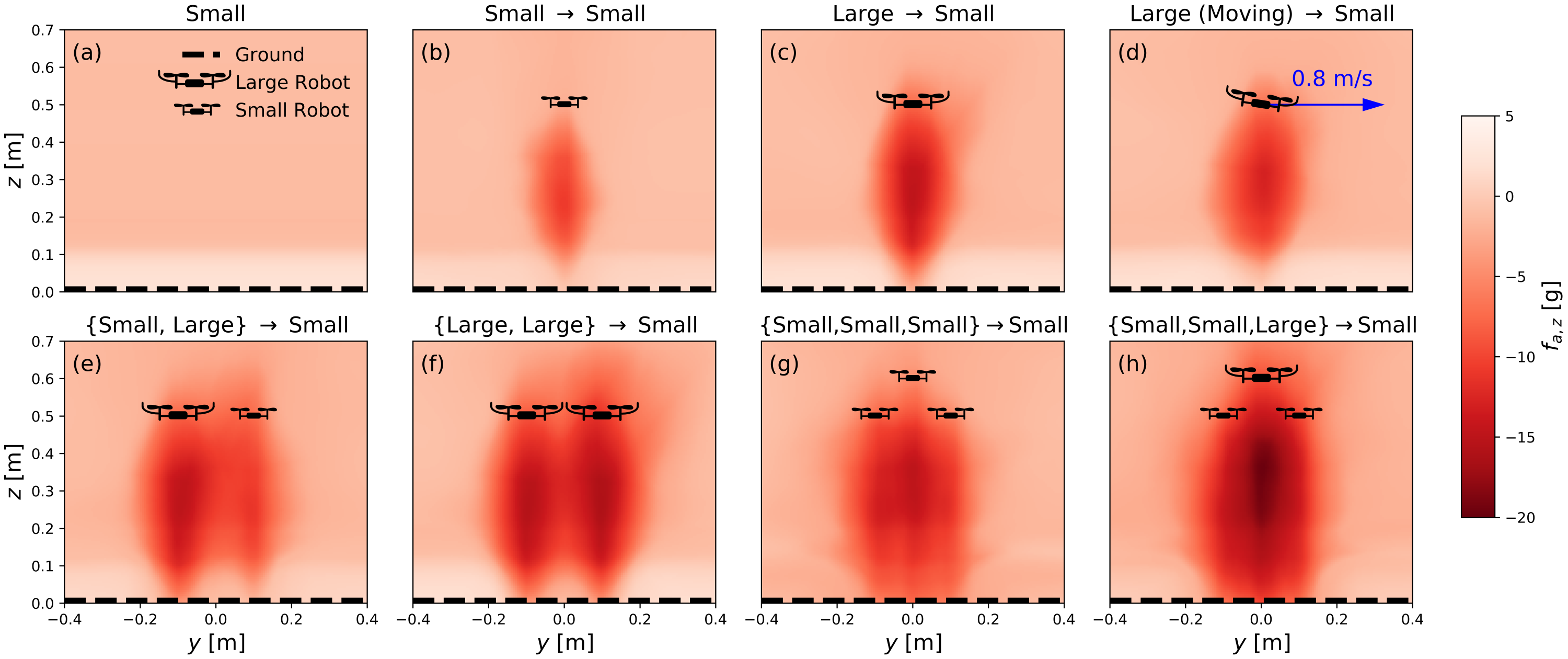}
\caption{\changed{$f_{a,z}$ prediction from the trained $\{\bm{\rho}_{\sm},\bm{\rho}_{\la},\bm{\phi}_{\sm},\bm{\phi}_{\la},\bm{\phi}_{\env}\}$ networks. Each heatmap gives the prediction of $f_{a,z}$ of a vehicle in different horizontal and vertical (global) positions. The (global) position of neighboring drones are represented by drone icons. See \cref{sec:ablation} for details.}}
\label{fig:vis}
\end{figure*}

Each scenario uses a trajectory with a duration around 1000 seconds. For each scenario, we equally split the total trajectory into 50 shorter pieces, where each one is about 20 seconds. Then we randomly choose 80\% of these 50 trajectories for training and 20\% for validation. 

Our DNN functions of $\bm{\phi}$ ($\bm{\phi}_{\sm},\bm{\phi}_{\la},\bm{\phi}_{\env}$) have four layers with architecture $6\rightarrow25\rightarrow40\rightarrow40\rightarrow H$, and our $\bm{\rho}$ DNNs ($\bm{\rho}_{\sm},\bm{\rho}_{\la}$) also have $L=4$ layers, with architecture $H\rightarrow40\rightarrow40\rightarrow40\rightarrow1$. We use the ReLU function as the activation operator, and we use PyTorch~\cite{pyTorch} for training and implementation of spectral normalization (see \cref{sec:spectral}) of all these five DNNs. During training we iterate all the data 20 times for error convergence. 

Note that $H$ is the dimension of the hidden state. To study the effect of $H$ on learning performance, we use three different values of $H$ and the mean validation errors for each $H$ are shown in \cref{tab:ablation}. Meanwhile, we also study the influence of the number of layers by fixing $H=20$ and changing $L$, which is the number of layers of all $\bm{\rho}$ nets and $\bm{\phi}$ nets. For $L=3$ or $L=5$, we delete or add a $40\rightarrow 40$ layer for all $\bm{\rho}$ nets and $\bm{\phi}$ nets, before their last layers. We repeat all experiments three times to get mean and standard deviation. As depicted in \cref{tab:ablation}, we found that the average learning performance (mean validation error) is not sensitive to $H$, but larger $H$ results in higher variance, \changed{possibly because using a bigger hidden space (larger $H$) leads to a more flexible encoding that is harder to train reliably.} In terms of the number of layers, four layers are significantly better than five (which tends to overfit data), and slighter better than three. To optimize performance, we finally choose $H=20$ and use four-layer neural networks, which can be efficiently evaluated on-board. \changed{We notice that $H$ and $L$ are the most important parameters, and the learning performance is not sensitive to other parameters such as the number of weights in intermediate layers.}

\Cref{fig:vis} depicts the prediction of $f_{a,z}$, trained with flight data from the 12 scenarios listed in \Cref{tab:data_collection}.
The color encodes the magnitude of $\hat{f}_{a,z}$ for a single small multirotor positioned at different global $(y,z)$ coordinates. 
The big/small black drone icons indicate the (global) coordinates of neighboring big/small multirotors, and the dashed line located at $z=0$ represents the ground. All quadrotors are in the same $x$-plane.
For example, in \cref{fig:vis}(e), one large quadrotor is hovering at $(y=-0.1,z=0.5)$ and one small quadrotor is hovering at $(y=0.1,z=0.5)$.
If we place a third small quadrotor at $(y=0,z=0.3)$, it would estimate $\hat{f}_{a,z}=\SI{-10}{g}$ as indicated by the red color in that part of the heatmap.
Similarly, in \cref{fig:vis}(a) the small multirotor only has the environment as a special neighbor. If the small multirotor is hovering at $(y=0,z=0.05)$, it would estimate $\hat{f}_{a,z}=\SI{5}{g}$, which is mainly from the ground effect.
All quadrotors are assumed to be stationary except for \cref{fig:vis}(d), where the one neighbor is moving at \SI{0.8}{m/s}.

We observe that the interaction between quadrotors is non-stationary and sensitive to relative velocity. In \cref{fig:vis}(d), the vehicle's neighbor is moving, and the prediction becomes significantly different from \cref{fig:vis}(c), where the neighbor is just hovering. \changed{To further understand the importance of relative velocity, we retrain neural networks neglecting relative velocity and the mean squared validation error degrades by $18\%$, from $6.42$ to $7.60$.}
We can also observe that the interactions are not a simple superposition of different pairs. \changed{For instance, \cref{fig:vis}(g) is significantly more complex than a simple superposition of \cref{fig:vis}(a) plus three (b), i.e., $\bm{\rho}_{\sm}(\bm{\phi}_{\env})+\bm{\rho}_{\sm}(\bm{\phi}_{\sm})+\bm{\rho}_{\sm}(\bm{\phi}_{\sm})+\bm{\rho}_{\sm}(\bm{\phi}_{\sm})$. The maximum gap between \cref{fig:vis}(g) and the superposition version is \SI{11.4}{g}.}
Moreover, we find that the ground effect and the downwash effect from a neighboring multirotor interact in an intriguing way. For instance, in \cref{fig:vis}(b), the downwash effect is ``mitigated'' as the vehicle gets closer to the ground.
Finally, we observe that the large quadrotors cause significantly higher interaction forces than the small ones (see \cref{fig:vis}(e)), which further emphasizes the importance of our heterogeneous modeling.

Note that in training we only have data from 1-3 vehicles (see \Cref{tab:data_collection}). Our approach can generalize well to a larger swarm system. In \cref{fig:vis}, predictions for a 4-vehicle team (as shown in \cref{fig:vis}(g,h)) are still reliable. Moreover, our models work well in real flight tests with 5 vehicles (see \cref{fig:plot5}) and even 16 vehicles (see \cref{fig:fig1}).

\begin{table}
\caption{Ablation analysis. Top: $L=4$ and $H$ varies. Bottom: $H=20$ and $L$ varies. The error is the mean squared error (MSE) between $f_{a,z}$ prediction and the ground truth.}
\label{tab:ablation}
\centering
\begin{tabular}{c|ccc}
$H$ & 10 & 20 & 40 \\ \hline
Validation Error & 6.70$\pm$0.05 & 6.42$\pm$0.18 & 6.63$\pm$0.35
\end{tabular}\\[2mm]
\begin{tabular}{c|ccc}
$L$ & 3 & 4 & 5 \\ \hline
Validation Error & 6.52$\pm$0.17 & 6.42$\pm$0.18 & 7.21$\pm$0.28 
\end{tabular}
\end{table}

\subsection{Motion Planning with Aerodynamics Coupling}
We implement \cref{alg:planning} in Python using PyTorch 1.5~\cite{pyTorch} for automatic gradient computation, CVXPY 1.0~\cite{diamond2016cvxpy} for convex optimization, and GUROBI 9.0~\cite{gurobi} as underlying solver.
To simulate the tracking performance of the planned trajectories, we also implement a \changed{nonlinear} controller, which uses the planned controls as feed-forward term.
We compare trajectories that were planned with a learned model of $f_{a,z}$ with trajectories without such a model (i.e., $f_{a,z}=0$) using \cref{alg:planning} with identical parameters.
At test time, we track the planned trajectories with our controller, and forward propagate the dynamics with our lea\changed{r}ned model of $f_{a,z}$.

\begin{figure*}
\includegraphics[width=\linewidth]{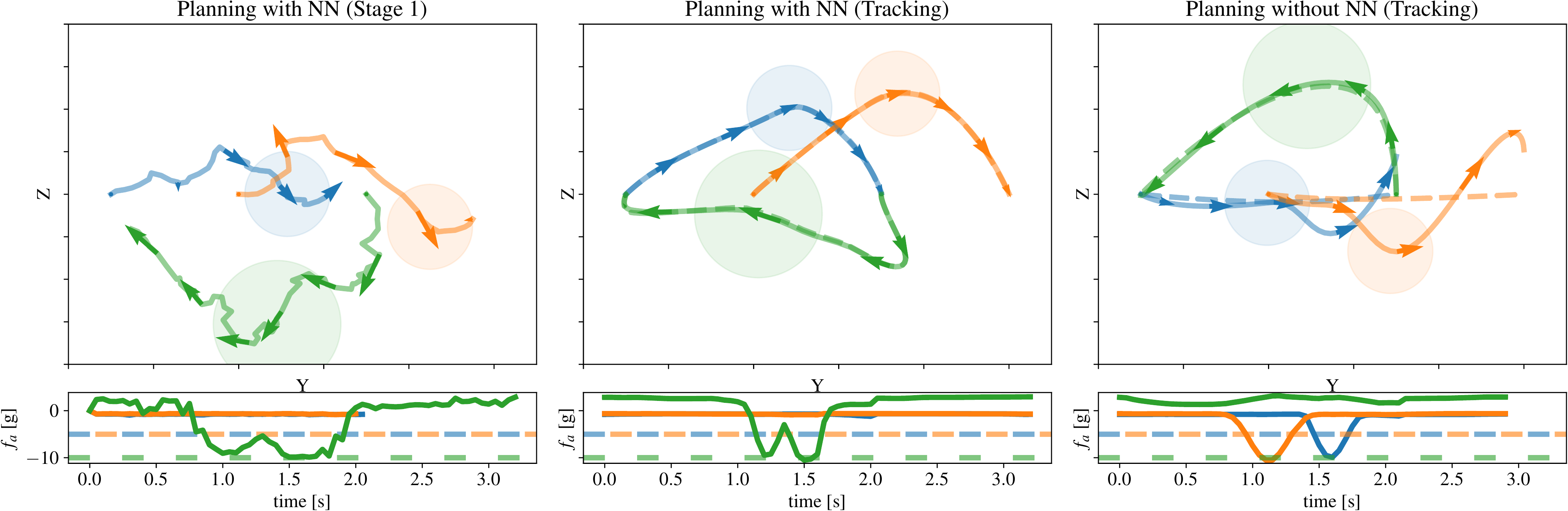}
\caption{Example motion planning result for a three-robot swapping task in 2D (blue and orange: small robots; green: large robot). \changed{Top row: $yz$-state space plot, where the arrows indicate the velocities every second, and the circles show the robot collision boundary shape at the middle of the task.} Bottom row: interaction force for each robot over time (dashed: desired limit per robot). 
Left: Sampling-based motion planning with neural network to compute trajectories where the large robots moves below the small robots.
Middle: Refined trajectories using SCP (dashed) and tracked trajectories (solid).
Right: Planned trajectories when ignoring interaction forces (dashed) and tracked trajectories (solid). In this case, a dangerous configuration is chosen where the large robot flies on top of the small robots, exceeding their disturbance limits of \SI{5}{g}.}
\label{fig:plot2}
\end{figure*}

We visualize an example in \cref{fig:plot2}, where two small and one large robots are tasked with exchanging positions.
We focus on the 2D case in the $yz$-plane to create significant interaction forces between the robots.
The first stage of \cref{alg:planning} uses sampling-based motion planning to identify the best homeomorphism class where the small multirotors fly on top of the large multirotor (the interaction forces would require more total energy the other way around). 
However, the robots do not reach their goal state exactly and motions are jerky (\cref{fig:plot2}, left).
The second stage uses SCP to refine the motion plan such that robots reach their goal and minimize the total control effort (\cref{fig:plot2}, middle).
The planned trajectory can be tracked without significant error and the interaction forces are very small for the two small quadrotors and within the chosen bound of $\SI{10}{g}$ for the large quadrotor.
We compare this solution to one where we do not consider the interaction forces between robots by setting $f_{a,z}=0$ in \cref{alg:planning}.
The planned trajectories tend to be shorter (\cref{fig:plot2}, right, dashed lines) in that case.
However, when tracking those trajectories, significant tracking errors occur and the interaction forces are outside their chosen bounds of $\SI{5}{g}$ for the small multirotors.

\begin{figure}[t]
\includegraphics[width=\linewidth]{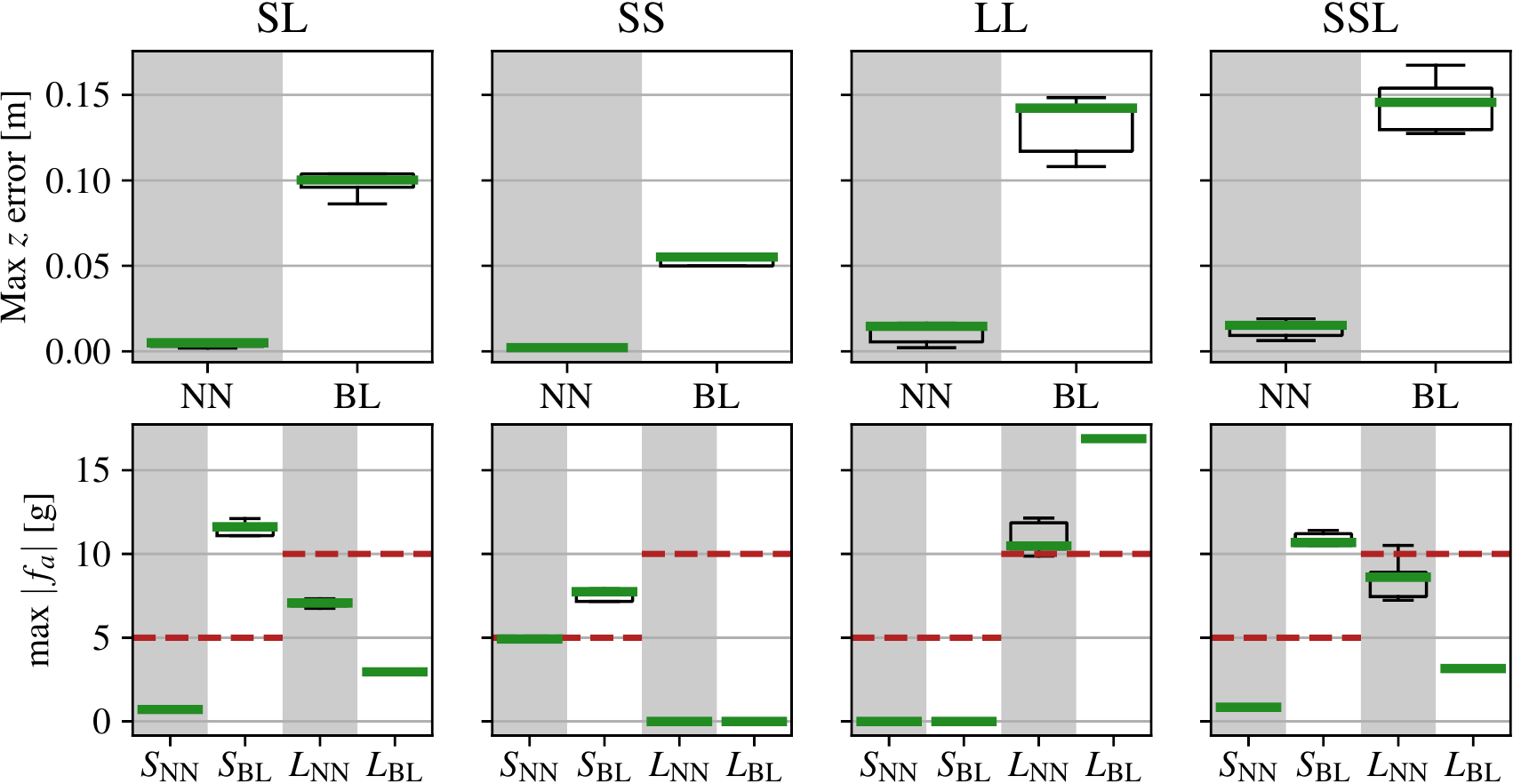}
\caption{Motion planning results for different scenarios (e.g., SSL refers to two small robots and one large robot) comparing planning without neural network (BL) and planning with neural network (NN) over 5 trials.
Top: Worst-case tracking error. Ignoring the interaction force can result in errors of over \SI{10}{cm}.
Bottom: Worst-case interaction force for small and large quadrotors. The baseline has significant violations of the interaction force bounds, e.g., the SL case might create interaction forces greater than $\SI{10}{g}$ for the small quadrotor.}
\label{fig:plot3}
\end{figure}

We empirically evaluated the effect of planning with and without considering interaction forces in several scenarios, see \cref{fig:plot3}. We found that ignoring the interaction forces results in significant tracking errors in all cases (top row).
While this tracking error could be reduced when using our interaction-aware control law, the interaction forces are in some cases significantly over their desired limit.
For example, in the small/large, small/small/large, and large/large cases, the worst-case interaction forces were consistently nearly double the limit (red line, bottom row).
In practice, such large disturbances \changed{can} cause instabilities or even a total loss of control, justifying the use of an interaction-aware motion planner.

\subsection{Control Performance in Flight Tests}

We study the flight performance improvements on swapping tasks with varying number of quadrotors.
For each case, robots are initially arranged in a circle when viewed from above but at different $z$-planes and are tasked with moving linearly to the opposite side of the circle in their plane.
During the swap, all vehicles align vertically at one point in time with vertical distances of \SIrange{0.2}{0.3}{m} between neighbors.
The tasks are similar, but not identical to the randomized swapping tasks used in \cref{exp:dataCollection} because different parameters (locations, transition times) are used.

\begin{figure}[t]
\includegraphics[width=\linewidth]{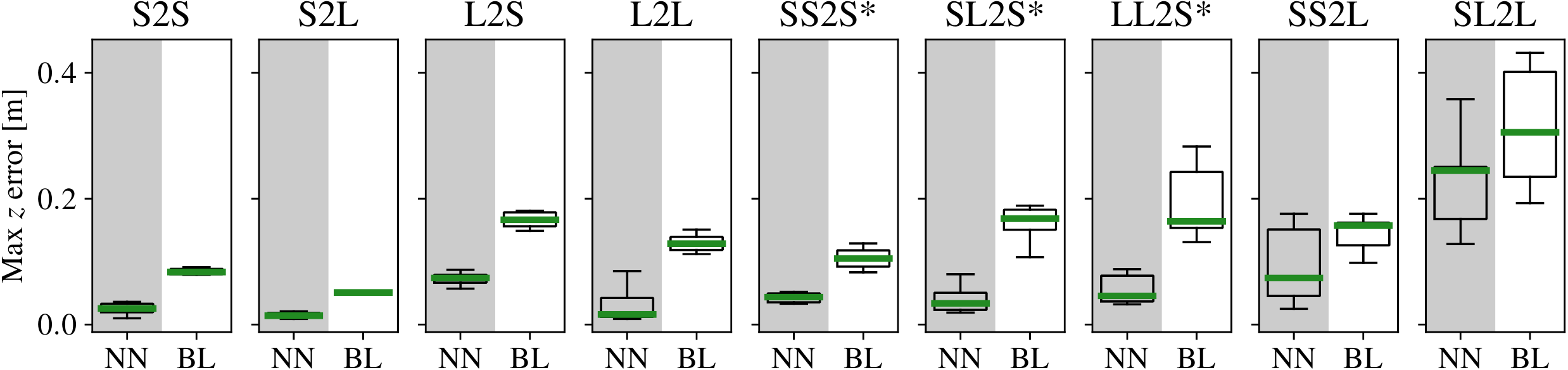}
\caption{Flight test results comparing our solution with learned interaction compensation (NN) with the baseline (BL) in different scenarios. For each case, robots are initially arranged in a circle when viewed from above but at different $z$-planes and are tasked with moving linearly to the opposite side of the circle in their plane. For each swap, we compute the worst-case $z$-error of the lowest quadrotor and plot the data over six swaps.}
\label{fig:plot4}
\end{figure}

Our results are summarized in \cref{fig:plot4} for various combinations of two and three multirotors, \changed{where we use ``XY2Z'' to denote the swap task with robots of type X and Y at the top and a robot of type Z at the bottom.} We compute a box plot with median (green line) and first/third quartile (box) of the maximum $z$-error (repeated over 6 swaps). In some cases, the downwash force was so large that we upgraded the motors of the small quadrotor to improve the best-case thrust-to-weight ratio to 2.6. Such modified quadrotors are indicated as ``S*''.
We also verified that the $x$- and $y$-error distributions are similar across the different controllers and \changed{omit} those numbers for brevity.

Our controller improves the \changed{median} $z$-error in all cases and in most cases this improvement is statistically significant. For example, in the ``L2S'' case, where a large multirotor is on top of a small multirotor for a short period of time, the median $z$-error is reduced from \SI{17}{cm} to \SI{7}{cm}.

\begin{figure}[t]
\includegraphics[width=\linewidth]{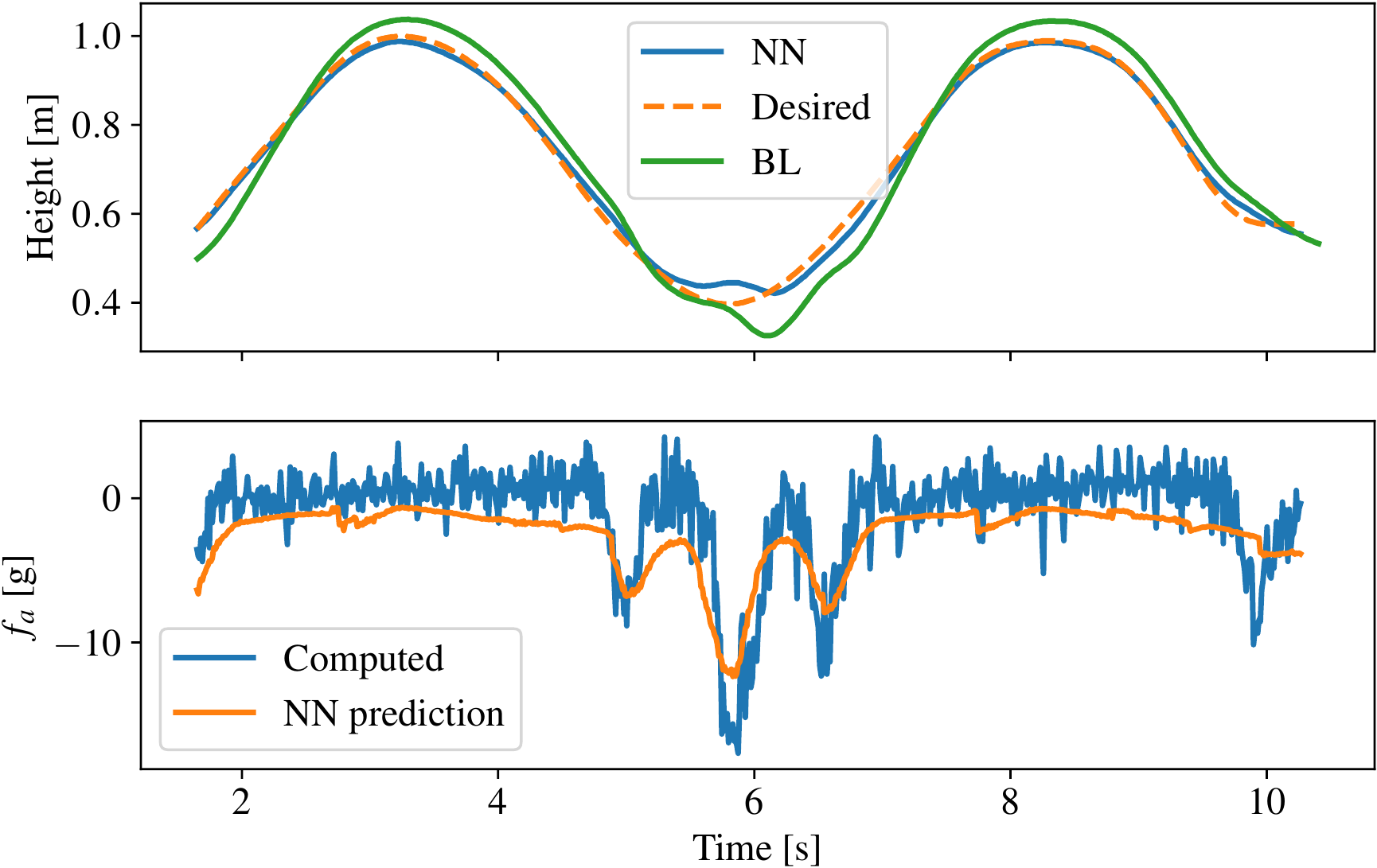}
\caption{Generalization to a team of five multirotors. Three small multirotor move in a vertical ring and two large multirotor move in a horizontal ring. The maximum $z$-error of a small multirotor in the vertical ring with powerful motors is reduced from \SI{10}{cm} to \SI{5}{cm} and $f_a$ is predicted accurately.}
\label{fig:plot5}
\end{figure}

To estimate the limits of our learning generalization, we test our approach on larger teams.
First, we consider a team of five robots, where two large robots move on a circle in the horizontal plane and three small robots move on  circle in the vertical plane such that the two circles form intertwined rings. In this case, the $f_{a,z}$ prediction is accurate and the maximum $z$-error can be reduced significantly using our neural network prediction, see \cref{fig:plot5} for an example.
Second, we consider a team of 16 robots moving on three intertwined rings as shown in \cref{fig:fig1}(b,c). Here, two large and four small robots move on an ellipsoid in the horizontal plane, and five robots move on circles in different vertical planes. In this case, robots can have significantly more neighbors (up to 15) compared to the training data (up to 2), making the prediction of $f_{a,z}$ relatively less accurate. However, the maximum $z$-error of a small multirotor in one of the vertical rings with powerful motors is still reduced from \SI{15}{cm} to \SI{10}{cm}.

\changed{We note that a conceptually-simpler method is to estimate and compensate for $f_{a,z}$ online without learning. However, online estimation will not only introduce significant delays, but also be very noisy especially in close-proximity flight. Our learning-based method has no delay (because it directly \emph{predicts} $f_{a,z}$ at the current time step), and considerably mitigates the noise due to the use of spectral normalization and delay-free filtering in the training process. In experiments, we observe that the online estimation and compensate method would quickly crash the drone.}

\section{Conclusion}
In this paper, we present \ns, a learning-based approach that enables close-proximity flight of heterogeneous multirotor teams.
Compared to previous work, robots can fly much closer to each other safely, because we accurately predict the interaction forces caused by previously unmodeled aerodynamic vehicle interactions.
To this end, we introduce \emph{heterogeneous deep sets} as an efficient and effective deep neural network architecture that only relies on relative positions and velocities of neighboring vehicles to learn the interaction forces between multiple quadrotors.
Our architecture also allows to model the ground effect and other unmodeled dynamics by viewing the physical environment as a special neighboring robot.
To our knowledge, our approach provides the first model of interaction forces between two or more multirotors.

We demonstrate that the learned interactions are crucial in two applications of close-proximity flight.
First, they can be used in multi-robot motion planning to compute trajectories that have bounded disturbances caused by neighboring robots and that consider platform limitations such as maximum thrust capabilities directly.
The resulting trajectories enable a higher robot density compared to existing work that relies on conservative collision shapes.
Second, we can compute the interaction forces in real-time on a small 32-bit microcontroller and apply them as additional feed-forward term in a novel delay-compensated nonlinear stable tracking controller.
Such an approach enables to reduce the tracking error significantly, if the maximum thrust capabilities of the robots are sufficient.

We validate our approach on different tasks using two to sixteen quadrotors of two different sizes and demonstrate that our training method generalizes well to a varying number of neighbors, is computationally efficient, and reduces the worst-case height error by a factor of two or better.


%





\ifCLASSOPTIONcaptionsoff
  \newpage
\fi



\printbibliography

%
\begin{IEEEbiography}[{\includegraphics[width=1in,height=1.25in,clip,keepaspectratio]{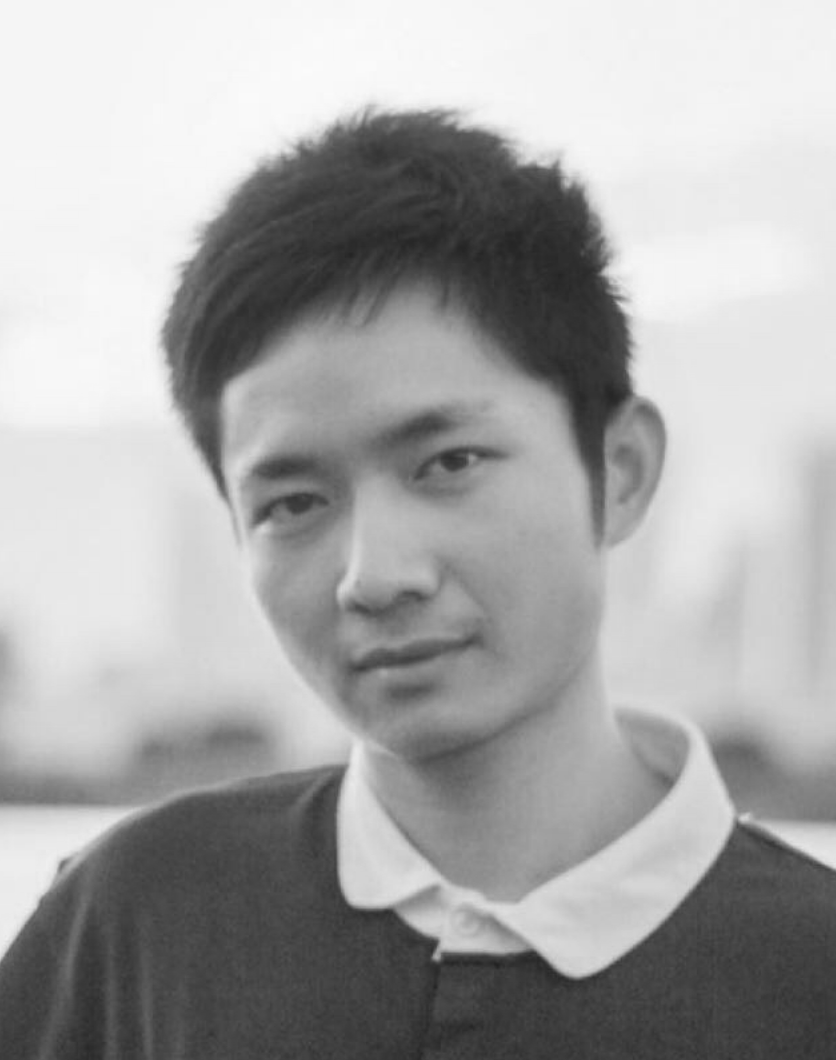}}]{Guanya Shi}
(S'18) is a Ph.D. student at the department of computing and mathematical sciences at the California Institute of Technology, USA. He holds a diploma in mechanical engineering (\textit{summa cum laude}) from Tsinghua University, China (2017). His research focuses on the intersection of machine learning and control theory with the applications in real-world complex systems such as robotics. He was the recipient of several awards, including the Simoudis Discovery Prize and the Qualcomm scholarship.
\end{IEEEbiography}

\begin{IEEEbiography}[{\includegraphics[width=1in,height=1.25in,clip,keepaspectratio]{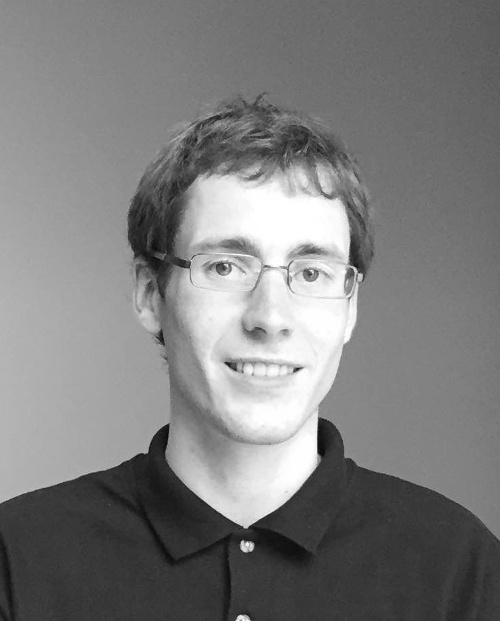}}]{Wolfgang Hönig}
(S'15--M'20) is an junior research group leader at TU Berlin, Germany. Previously, he was a postdoctoral scholar at the California Institute of Technology, USA. He received the diploma in Computer Science from TU Dresden, Germany in 2012, and the M.S. and Ph.D. degrees from the University of Southern California (USC), USA in 2016 and 2019, respectively.
His research focuses on enabling large teams of physical robots to collaboratively solve real-world tasks, using tools from informed search, optimization, and machine learning.
Dr. Hönig has been the recipient of several awards, including Best Paper in Robotics Track for a paper at ICAPS 2016 and the 2019 Best Dissertation Award in Computer Science at USC.
\end{IEEEbiography}


\begin{IEEEbiography}[{\includegraphics[width=1in,height=1.25in,clip,keepaspectratio]{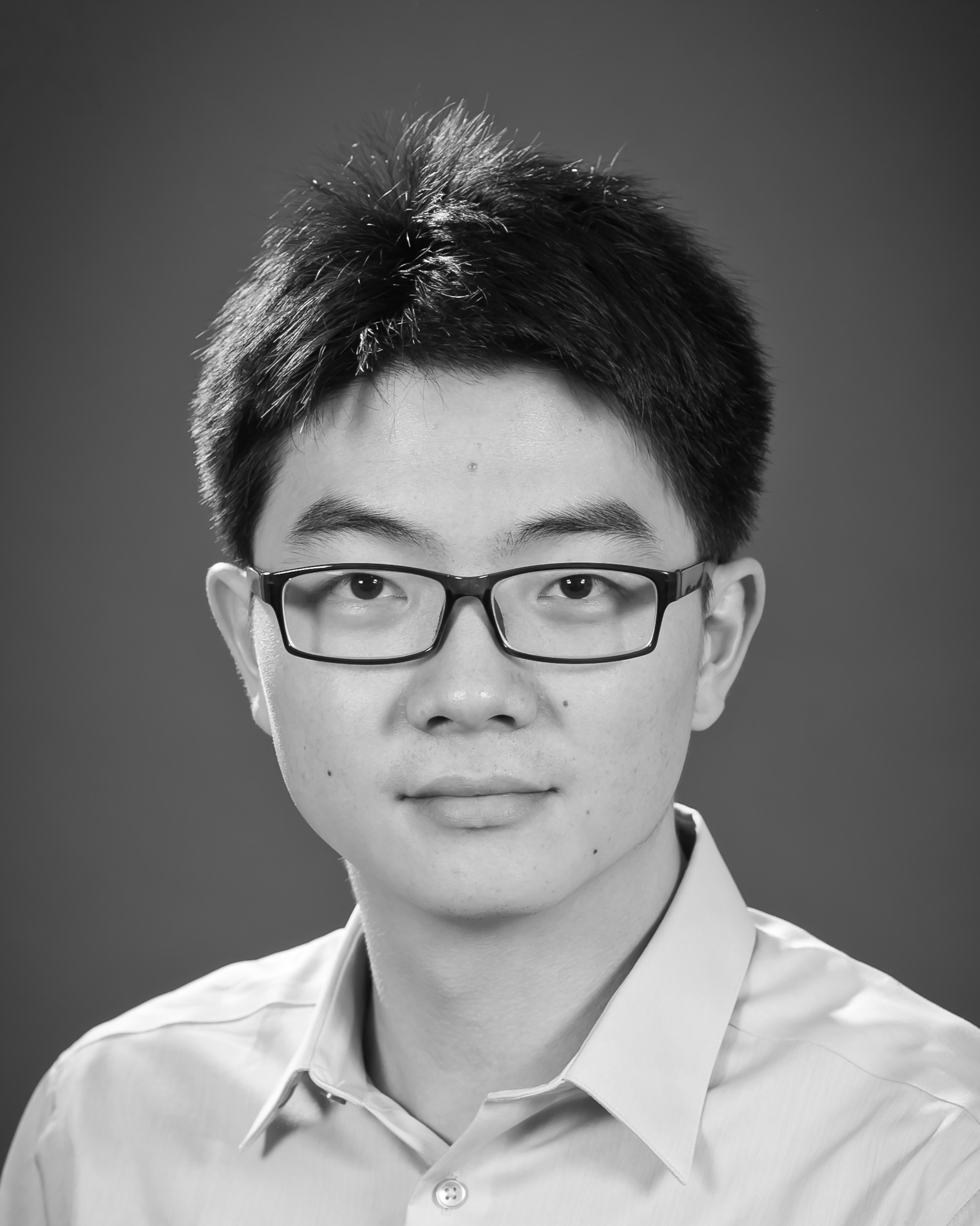}}]{Xichen Shi}(S'13--M'20) is a software engineer at Waymo, an autonomous driving vehicle company. He received a Ph.D. in Space Engineering from California Institute of Technology, USA in 2021, and a B.S in Aerospace Engineering (Highest Honors) from University of Illinois at Urbana-Champaign, USA in 2013. His research focuses on intelligent control systems for fixed-wing and multirotor aerial robots.
\end{IEEEbiography}

\begin{IEEEbiography}[{\includegraphics[width=1in,height=1.25in,clip,keepaspectratio]{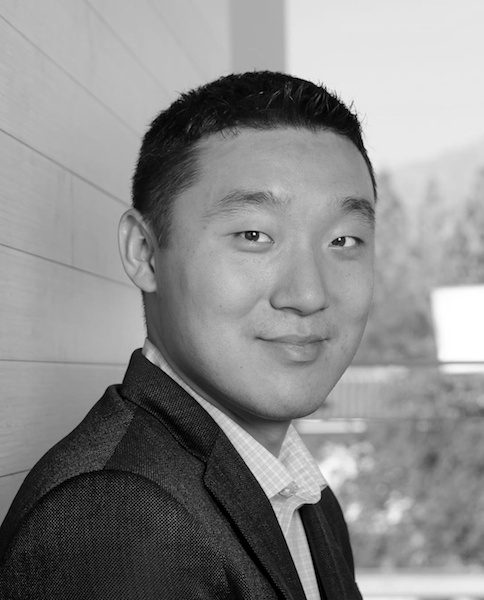}}]{Yisong Yue} is a professor of Computing and Mathematical Sciences at the California Institute of Technology. He was previously a research scientist at Disney Research and a postdoctoral researcher at Carnegie Mellon University. He received a Ph.D. from Cornell University and a B.S. from the University of Illinois at Urbana-Champaign.
Dr. Yue's research interests are centered around machine learning and has been applied to information retrieval, data-driven animation, behavior analysis, protein engineering, and learning-accelerated optimization.
Dr. Yue was the recipient of several awards and honors, including the Best Paper Award at ICRA 2020, the Best Student Paper Award at CVPR 2021, the Best Paper Nomination at WSDM 2011, ICDM 2014, SSAC 2017 and RA-L, Best Reviewer at ICLR 2018, and the Okawa Foundation Grant Recipient, 2018.
\end{IEEEbiography}

\begin{IEEEbiography}[{\includegraphics[width=1in,height=1.25in,clip,keepaspectratio]{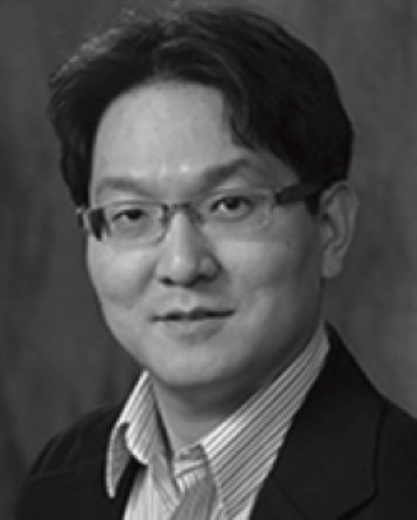}}]{Soon-Jo Chung} (M'06--SM'12) is the Bren Professor of Aerospace and Control and Dynamical Systems and a Jet Propulsion Laboratory Research Scientist in the California Institute of Technology, USA. He was with the faculty of the University of Illinois at Urbana-Champaign during 2009–2016.
He received the B.S. degree (\textit{summa cum laude}) in aerospace engineering from the Korea Advanced Institute of Science and Technology, South Korea, in 1998, and the S.M. degree in aeronautics and astronautics and the Sc.D. degree in estimation and control from Massachusetts Institute of Technology, USA, in 2002 and 2007, respectively.
His research interests include spacecraft and aerial swarms and autonomous aerospace systems, and in particular, on the theory and application of complex nonlinear dynamics, control, estimation, guidance, and navigation of autonomous space and air vehicles.
Dr. Chung was the recipient of the UIUC Engineering Deans Award for Excellence in Research, the Beckman Faculty Fellowship of the UIUC Center for Advanced Study, the U.S. Air Force Office of Scientific Research Young Investigator Award, the National Science Foundation Faculty Early Career Development Award, a 2020 Honorable Mention for the IEEE RA-L Best Paper Award, and three Best Conference Paper Awards from the IEEE and AIAA. He is an Associate Editor of IEEE T-AC and AIAA JGCD. 
He was an Associate Editor of IEEE T-RO, and the Guest Editor of a Special Section on Aerial Swarm Robotics published in the IEEE T-RO. 
He is an Associate Fellow of AIAA.
\end{IEEEbiography}




\end{document}